\documentclass[11pt, twoside]{article}
\usepackage[margin=1in]{geometry}

\usepackage[utf8]{inputenc} 
\usepackage[T1]{fontenc}    
\usepackage{hyperref}       
\usepackage{booktabs}       
\usepackage{amsfonts}       
\usepackage{nicefrac}       
\usepackage{microtype}      
\usepackage{color}
\usepackage{mathtools}
\usepackage{subcaption}
\usepackage{amsthm,amsmath,amssymb,graphicx,url}%
\usepackage{epstopdf}
\usepackage{enumitem}
\setlist{nosep,after=\vspace{\baselineskip}}
\usepackage{algorithm,algpseudocode}
\DeclareMathOperator*{\argmax}{arg\,max}

\usepackage{authblk} 

 \newcommand{\card}{\text{card} } 

\newtheorem{definition}{Definition}
\newtheorem{lemma}{Lemma}
\newtheorem{corollary}{Corollary}
\newtheorem{proposition}{Proposition}

\newtheorem{theorem}{Theorem}
\newtheorem{assumption}{Assumption}
\newtheorem{conjecture}{Conjecture}

\def\ci{\perp\!\!\!\perp}

\def\NP{\textsc{NP}}
\newcommand{\mat}[1]{\mathbf{#1}}
\newcommand{\pr}[1]{\mathbb{P}(#1)}
\newcommand*{\qedS}{\hfill\ensuremath{\square}}%

\begin{document}
%
\date{\today}
\title{Entropic Causal Inference}
\author[1,*]{Murat Kocaoglu}
\author[1,\textdagger]{Alexandros G. Dimakis}
\author[1,\textdaggerdbl]{Sriram Vishwanath}
\author[2,\textsection]{Babak Hassibi}
\affil[1]{\small Department of Electrical and Computer Engineering, The University of Texas at Austin, USA}
\affil[2]{\small Department of Electrical Engineering, California Institute of Technology, USA}
\affil[ ]{\small \textit \textsuperscript{*} mkocaoglu@utexas.edu \textsuperscript{\textdagger}dimakis@austin.utexas.edu  \textsuperscript{\textdaggerdbl}sriram@ece.utexas.edu \textsuperscript{\textsection}hassibi@systems.caltech.edu}
\renewcommand\Authands{ and }

\maketitle
\begin{abstract}
We consider the problem of identifying the causal direction between two discrete random variables using observational data. 
Unlike previous work, we keep the most general functional model but make an assumption on the unobserved exogenous variable: 
Inspired by Occam's razor, we assume that the exogenous variable is \textit{simple} in the true causal direction. 
We quantify simplicity using R\'enyi entropy. 
Our main result is that, under natural assumptions, if the exogenous variable has low $H_0$ entropy (cardinality) in the true direction, it must have high $H_0$ entropy in the wrong direction. We establish several algorithmic hardness results about 
estimating the minimum entropy exogenous variable. We show that the problem of finding the exogenous variable with minimum entropy is equivalent to the problem of finding minimum joint entropy given $n$ marginal distributions, also known as minimum entropy coupling problem. We propose an efficient greedy algorithm for the minimum entropy coupling problem, that for $n=2$ provably finds a local optimum. This gives a greedy algorithm for finding the exogenous variable with minimum $H_1$ (Shannon Entropy).
Our greedy entropy-based causal inference algorithm has similar performance to the state of the art additive noise models in real datasets. One advantage of our approach is that we make no use of the values of random variables but only their distributions. Our method can therefore be used for causal inference for both ordinal and also categorical data, unlike additive noise models.  
\end{abstract}

\section{Introduction}

Causality has been studied under several frameworks including potential outcomes~\cite{Rubin1974}
and structural equation modeling~\cite{Pearl2009}. Under the Pearlian framework~\cite{Pearl2009} it is possible to discover some causal directions between variables using only observational data with conditional independence tests. The PC algorithm~\cite{Spirtes2001} and its variants fully characterize which 
causal directions can be learned in the general case. For large graphs, GES algorithm~\cite{Chickering2002} provides a score-based test to greedily identify the highest scoring causal graph given the data. Unfortunately, these approaches do not guarantee the recovery of true causal direction between every pair of variables, since typically data could be generated by several statistically equivalent causal graphs. 

A general solution to the causal inference problem is to conduct experiments, also called interventions.  
An intervention forces the value of a variable without affecting the other system variables. This removes the effect of its causes, effectively creating a new causal graph. These changes in the causal graph create a post-interventional distribution among variables, which can be used to identify some additional causal relations in the original graph. The procedure can be applied repeatedly to fully identify any causal graph \cite{Hauser2012a}, \cite{Hauser2012b}, \cite{Hyttinen2013}, \cite{Shanmugam2015}.

Unfortunately, for many problems, it can be very difficult to create interventions since they require additional experiments after the original data collection. Researchers would still like to discover causal relations between variables using only observational data, using so-called data-driven causality. Several recent works \cite{Chen2014,Shajarisales2015} have developed such methods. To be able to make any conclusions on causal directions in this case, additional assumptions must be made about the mechanisms that generate the data.

In this paper we focus on the simplest causal discovery problem that involves only two variables. 
The two causal graphs $X\rightarrow Y$ and $X\leftarrow Y$ are statistically indistinguishable so conditional independence tests cannot make any causal inference from observational data without interventions. Statistical indistinguishability easily follows from the fact that any joint distribution on two variables $p(x,y)$ can be factorized both as $p(x)p(y/x)$ and $p(y)p(x/y)$.

The most popular assumption for two-variable data-driven causality is the additive noise model (ANM) \cite{Shimizu2006}. In ANM, any outside factor is assumed to affect the effect variable additively, which leads to the equation $Y = f(X)+E,E\ci X$. Although restrictive, this assumption leads to strong theoretical guarantees in terms of identifiability, and provides the state of the art accuracy in real datasets. 
~\cite{Shimizu2006} showed that if $f$ is linear and the noise is non-Gaussian the causal direction is identifiable. 
~\cite{Hoyer2008} showed that when $f$ is non-linear, irrespective of the noise, identifiability holds in a non-adverserial setting of system parameters. ~\cite{Peters2011} extended ANM to discrete variables.

Another approach is to exploit the postulate that the cause and mechanism are in general independently assigned by nature. The notion of \emph{independence} here is vague and one needs to assign maps, or conditional distributions to random variables to argue about independence of cause and mechanism. In this direction an information-geometry based approach is suggested \cite{Janzing2012}. Independence of cause and mechanism is captured by treating the log-slope of the function as a random variable, and assuming that it is independent from the cause. In the case of a deterministic relation $Y=f(X)$, there are theoretical guarantees on identifiability. However, this assumption is restrictive for real data.

Previous work exploited these two ideas, additive noise, and independence of cause and mechanism, to draw data-driven causal conclusions about problems in a diverse range of areas from astronomy to neuroscience \cite{Shajarisales2015}, \cite{Scholkopf2015}. \cite{Shajarisales2015} uses the same idea that the cause and effect are independent in the time series of a linear filter. They suggest the spectral independence criterion, which is robust to time shifts. \cite{Chen2014} uses kernel space embeddings with the assumption that the cause distribution $p(x)$ and mechanism $p(y|x)$ are selected independently to distinguish cause from effect. 

As noted by \cite{Chen2014}, although conceptually proposed before, using Kolmogorov complexity of the factorization of the joint distribution $p(y|x)p(x)$ and $p(x|y)p(y)$ as a criterion for deciding causal direction has not been used successfully until now. 

The use of information theory as a tool for causal discovery is currently gaining increasing attention. This is through different appoaches, e.g., for time-series data, Granger causality and Directed Information can be used \cite{Granger1969,Etesami2016,Quinn2015}, see also \cite{Munich2016}. However, researchers have not used entropy as a measure of simplicity in the causal discovery literature, probably because the entropies $H(Y|X)$ and $H(X|Y)$ do not give us any more information than $H(X)$ and $H(Y)$, due to the symmetry $H(Y)+H(X|Y)=H(X)+H(Y|X)$. In our work, as we will explain, we minimize $H(E)$ which initially sounds similar, \textit{but is fundamentally different} from $H(Y|X)$.  Entropy has found some additional uses in the causality literature recently: In \cite{Gao2016}, authors use maximum mutual information between $X,Y$ in order to quantify the causal strength of a known causal graph.

The work that is most similar to ours in spirit is \cite{Mooij2010}, which also drops the additive noise assumption. Their approach and setup are different in many ways: Authors work with continuous data. To be able to handle this generic form, they have to make strong assumptions on the exogenous variable, function, and distribution of the cause: \cite{Mooij2010} assume that the exogenous variable is a standard Gaussian, a Gaussian mixture prior for the cause, and a Gaussian process as the prior of the function.


\subsection{Our contributions}
In this paper, we propose a novel approach to the causal identifiability problem for discrete variables. Similar to \cite{Mooij2010}, we keep the most general functional model, but only put an assumption on the exogenous (background) variable. Based on Occam's razor, we employ a simplicity assumption on the unobserved exogenous variable. We use R\'enyi entropy, which is defined as $H_a(X) = \frac{1}{1-a}\log{\left(\sum_i p_i^a \right)}$, for a random variable $X$ with state probabilities $p_i$. We focus on two special cases of R\'enyi entropy: $H_0$, which corresponds to the logarithm of the number of states, and $H_1$ which corresponds to Shannon entropy, but our framework can be extended. 

Specifically, if the true causal direction is $X\rightarrow Y$,  then the random variable $Y$ is an arbitrary function of $X$ and an exogenous variable $E$: $Y= f(X,E)$ where $E$ is independent from the cause $X$. Our key assumption is that the exogenous variable $E$ is \textit{simple}, i.e., has low R\'enyi entropy. The postulate is that for any model in the wrong direction 
$X= f'(Y,\tilde{E})$, the exogenous variable $\tilde{E}$ has high R\'enyi entropy. We are able to prove this result for the $H_0$ special case of R\'enyi entropy, assuming generic distributions for $X,Y$. Furthermore, we empirically show that using $H_1$ Shannon entropy we obtain practical causality tests that work with high probability in synthetic datasets and that slightly outperforms the previous state of the art in real datasets. 

Our assumption is an entropic interpretation of Occam's razor, motivated by what $E$ represents in the causal model.
The exogenous variable captures the combined effect of all the variables not included in the system model, which affect the distribution of $Y$. Our causal assumption can be stated as \emph{``there should not be too much complexity not included in the causal model"}.
For $a\rightarrow 1$, \textit{i.e.}, Shannon entropy, $H(X)+H(E)$, $H(Y)+H(\tilde{E})$ are the number of random bits required to generate an input for the causal system $X\rightarrow Y$ and $X\leftarrow Y$, respectively. The simplest explanation of an observed joint distribution, i.e., the direction which requires nature to generate smaller number of random bits is selected as the true causal model. More precisely we have the following: 
\begin{assumption}\label{ass:causal}
Entropy of the exogenous variable $E$ is small in the true causal direction.
\end{assumption}

The notions of simplicity that we consider are $H_0$, which is log-cardinality, and $H_1$, which is Shannon entropy. One significant advantage of using Shannon entropy as a simplicity metric is that it can be estimated more robustly in the presence of measurement errors, unlike cardinality $H_0$. 

We prove an identifiability result for $H_0$ entropy, i.e., cardinality of $E$: If the probability values are not adversarially chosen, for most functions, the true causal direction is identifiable under Assumption \ref{ass:causal}. Based on experimental evidence, we conjecture that a similar identifiability result must hold for Shannon entropy $H_1$. 

To use our framework we need algorithms that explain a dataset by finding an exogenous variable $E$ with minimum cardinality $H_0$ and minimum Shannon entropy $H_1$. Since the entropies of $X$ and $Y$ can be very different, any metric to determine
the true causal direction cannot only consider the entropy of the exogenous variable without incorporating the entropy of the cause. 
We explain the exogenous variable in both directions and declare the causal direction to be the one with the smallest joint entropy $H_a(X)+H_a(E)$ versus $H_a(Y)+H_a(\tilde{E})$. Our method can be applied for any R\'enyi entropy $H_a$ but in this paper we only use $a=0$ and $a=1$.

Unfortunately, minimizing $H_0(E)$ seems very hard for real datasets since it offers no noise robustness. For Shannon entropy we can do much better for real data. The first step in obtaining a practical algorithm is showing that the minimum $H_1$ explanation is equivalent to the following problem:
For $n$ random variables with given marginal distributions, find a joint distribution with the minimum Shannon entropy that is consistent with the given marginals. This problem is called the minimum Shannon entropy coupling and is known to be NP hard \cite{Kovacevic2012}. We propose a greedy approximation algorithm for this problem that empirically performs very well. We also prove that, for $n=2$, our algorithm always produces a local minimum.

In summary our contributions in this paper include: 
\begin{itemize}
\item We show identifiability for generic low-entropy causal models under Assumption \ref{ass:causal} with $H_0$.
\item We show that the problems of identifying the minimum cardinality ($H_0$) exogenous variable, and identifying the minimum Shannon entropy ($H_1$) exogenous variable given a joint distribution are both $\NP$ hard.
\item We design a novel greedy algorithm for the minimum entropy coupling problem, which turns out to be equivalent to the problem of finding exogenous variable with minimum $H_1$ entropy. 
\item We empirically validate the conjecture that the causal direction is identifiable under Assumption \ref{ass:causal} with $H_1$, using experiments on synthetic datasets.
\item We empirically show that our causal inference algorithm based on Shannon entropy minimization has slightly better performance than the existing best algorithms on a real causal dataset. Interestingly, our algorithm uses only the probability distributions rather than the actual values of the random variables, and hence is applicable to categorical variables. 
\end{itemize}

\subsection{Background and Notation}
A tuple $\mathcal{M} = (X,U,\mathcal{F},D,p)$ is a causal model when, $1)$ $\mathcal{F} = \{f_i\}$ are deterministic functions, $2)$ $X = \{X_i\}$ are a set of endogenous (observed) variables $U = \{U_i\}$ are a set of exogenous (latent) variables with $X_i = f_i(Pa_i,U_i),\forall i$ where $Pa_i$ are the endogenous parents and $U_i$ is the exogenous parent of $X_i$ in directed acyclic graph $D$, $3)$ $U$ are mutually independent with respect to $p$. The observable variable set $X$ has a joint distribution implied by the distributions of $U$, and the functional relations $f_i$. $D$ is then a Bayesian network for the induced joint distribution of endogenous variables. A standard assumption employed in Pearl's model \emph{causal sufficiency} is also used here: Every exogenous variable is a direct parent of at most one endogenous variable.

In this paper, we consider a simple two variable causal system which contains only two endogenous variables $X,Y$. Assume $X$ causes $Y$, which is represented as $X\rightarrow Y$. The model is determined only by one exogenous variable $E$, and a function $f$, where $Y = f(X,E)$.  The probability distribution of $X$ and $E$, and $f$ determines the  distribution of $Y$. This model is shown by the tuple $\mathcal{M} = (\{X,Y\},E,f,X\rightarrow Y,p_{X,E})$. Notice that we do not assign an exogenous variable to $X$, since it is the source node in the graph.

We denote the set $\{1,2,\cdots,n\}$ by $[n]$. $\sum_ix_i$ is meant to run through every possible index. $\log$ refers to the logarithm base 2. For two variables $X,Y$, $\mat{Y|X}$ and $\mat{X|Y}$ denote the conditional probability distribution matrices, i.e., $\mat{Y|X}(i,j) = p(y=i|x=j)$ and $\mat{X|Y}(i,j) = p(x = i|y = j)$. The statistical independence of two random variables $X$ and $E$ are shown by $X\ci E$. For notational convenience, probability distribution of random variable $X$ is shown by $p(x)$ as well as $p_X(x)$. $\mat{x}$ shows the distribution of $X$ in vector form , i.e., $x_i = \mat{x}(i)=\pr{X=i}$. $n-1$ simplex is the set of points $x$ in $n$ dimensional Euclidean space that satisfy $\sum_ix(i)=1$. $\card$ is the cardinality of a set.

\section{Causal Model with Minimum Cardinality Exogenous Variable}
Consider the causal model $\mathcal{M} = (\{X,Y\},E_0,f_0,X\rightarrow Y,p_{X,E})$. The task is to identify the underlying causal graph $X\rightarrow Y$ using independent identically distributed samples $\{(x_i,y_i)\}_i$. Assuming causal sufficiency, this task reduces to deciding whether $X$ causes $Y$ or $Y$ causes $X$. To isolate the identifiability problem from estimation errors 
due to finite samples, we assume that the joint distribution of $(X,Y)$ is available. Most proofs are deferred to the Appendix.

One way to identify that $X$ causes $Y$ is by showing that although there exists a function $f$ and random variable $E$ with $Y = f(X,E), X\ci E$, there is no function, random variable pair $(g,\tilde{E})$ such that $X = g(Y,\tilde{E}),Y\ci \tilde{E}$. However, without more assumptions, this is not possible: For any joint distribution one can find valid causal models for both $X\rightarrow Y,X\leftarrow Y$. This is widely known, although for completeness, we provide a proof (Lemma \ref{lem:unidentifiability} in the Appendix). 

Even when the true causal graph is known, one can create different constructions of $f,E$ with $Y=f(X,E),X\ci E$. There is no way to distinguish the true causal model. However, even though we cannot recover the actual function and the exogenous variable, we can still show identifiability.

First, we give an equivalent characterization of a causal model on two variables. 

\begin{definition}[\bf Block Partition Matrices] Consider a matrix $\mat{M}\in\{0,1\}^{n^2\times m}$. Let $\mat{m}_{i,j}$ represent the $i+(j-1)n$ th row of $\mat{M}$. Let $S_{i,j}=\{k\in[m]: \mat{m}_{i,j}(k)\neq 0\}$. $\mat{M}$ is called a block partition matrix if it belongs to $\mathcal{C} \coloneqq \{\mat{M}: \mat{M}\in \{0,1\}^{n^2\times m}, \bigcup_{i\in[n]} S_{i,j} = [m], S_{i,j}\cap S_{l,j}=\emptyset, \forall i\neq l \}$. 
\end{definition}

$\mathcal{C}$ thus stands for $0,1$ matrices with $n^2$ rows and $m$ columns where each block of $n$ rows correspond to a partitioning of the set $[m]$. We make the following key observation:
\begin{lemma}
\label{lem:characterization}
Given discrete random variables $X,Y$ with distribution $p(x,y)$, $\exists$ a causal model $\mathcal{M} = (\{X,Y\},E, f, X\rightarrow Y, p_{X,E})$, $E\in\mathcal{E}$ with $\card(\mathcal{E}) = m$ if and only if $\exists \mat{M}\in\mathcal{C}, \mat{e}\in \mathbb{R}_+^{m}$ with $\sum_i \mat{e}(i) = 1$ that satisfy $vec(\mat{Y|X})=\mat{Me}$.
\end{lemma}

In other words, the existence of a causal pair $X\rightarrow Y$ is equivalent to the existence of a block partition matrix $\mat{M}$ and a vector $\mat{e}$ of proper dimensions with $vec(\mat{Y|X})=\mat{Me}$. 

For simplicity, assume $\lvert\mathcal{X}\rvert = \lvert\mathcal{Y}\rvert = n$. We later remove this constraint. We first show that any joint distribution can be explained using a variable $E$ with $n(n-1)+1$ states.

\begin{lemma}[\textbf{Upper Bound on Minimum Cardinality of $E$}]
\label{lem:decompUpper}
Let $X\in\mathcal{X},Y\in\mathcal{Y}$ be two random variables with joint probability distribution $p_{X,Y}(x,y)$, where $\lvert \mathcal{X}\rvert = \lvert \mathcal{Y}\rvert = n$. Then $\exists$ a causal model $Y=f(X,E),X\ci E$ that induces $p_{X,Y}$, where $E$ has support size $n(n-1)+1$.
\end{lemma}

We can show that, if the columns of $\mat{Y|X}$ are uniformly sampled points in the $n-1$ dimensional simplex, then $n(n-1)$ states are also necessary for $E$ (see Proposition \ref{prop:helper} in the Appendix). This shows, unless designed by nature through the causal mechanism, exogenous variable cannot have small cardinality. Based on this observation, the hope is to prove that in the wrong causal direction, say $X\rightarrow Y$ and we find an $\tilde{E}\ci Y$ such that $X=g(Y,\tilde{E})$ for some $g$, the exogenous variable $\tilde{E}$ has to have large cardinality. In the next section, we show this is actually through, under mild conditions on $f$.

\subsection{Identifiability for \texorpdfstring{$H_0$}{} entropy}
\label{subsec:identifiability}
In a causal system $Y = f(X,E)$, nature chooses the random variables $X,E$, and function $f$, and the conditional probability distributions are then determined by these. We are interested in the cardinality of variables $\tilde{E}\ci X$ in the wrong causal direction $X = g(Y,\tilde{E})$. Considering $\mat{X|Y}$, we can show that the same lower bound of $n(n-1)$ still holds despite nature now chooses $E$ and $X$ randomly, rather than choosing the columns of $\mat{X|Y}$ directly. A mild assumption on $f$ is needed to avoid degenerate cases (For counterexamples see the appendix).

\begin{definition}[\bf Generic Function]
Let $Y = f(X,E)$ where variables $X,Y,E$ have supports $\mathcal{X},\mathcal{Y},\mathcal{E}$, respectively. Let $S_{y,x} = f_x^{-1}(y)\subset \mathcal{E}$ be the inverse map for $x,e$, i.e., $S_{y,x} = \{e\in\mathcal{E}: y = f(x,e)\}$. A function $f$ is called ``generic", if for each $(x_1,x_2,y)$ triple $f_{x_1}^{-1}(y) \neq f_{x_2}^{-1}(y)$ and for every $(x,y)$ pair $f_x^{-1}(y)\neq \emptyset$.
\end{definition}

In other words $f$ is called generic if $y^{th}$ row in the $x_1^{th}$ block of matrix $\mat{M}$ in the decomposition $vec(\mat{Y|X}) = \mat{Me}$ is different from $y^{th}$ row in the $x_2^{th}$ block, and both are nonzero. This is not a restrictive condition, for example if $p(y|x)$ are all different, no two rows of $\mat{M}$ can be the same. For any given conditional distribution, if the probabilities are perturbed by arbitrarily small continuous noise, the corresponding $f$ will be generic almost surely. We have the following main identifiability result:
\begin{theorem}[\textbf{Identifiability}]
\label{thm:main}
Consider the causal model $\mathcal{M}=(\{X,Y\},E_0,f_0,X\rightarrow Y,p_{X,E_0})$ where the random variables $X,Y$ have $n$ states, $E_0\ci X$ has $\theta$ states and $f$ is a generic function . 

If the distributions of $X$ and $E$ are uniformly randomly selected from the $n-1$ and $\theta -1 $ simplices, then with probability 1, any $\tilde{E}\ci Y$ that satisfies $X = g(Y,\tilde{E})$ for some deterministic function $g$ has cardinality at least $n(n-1)$.
\end{theorem}
Theorem \ref{thm:main} implies that the causal direction is identifiable, when the exogenous variable has cardinality $<n(n-1)$:
\begin{corollary}
\label{cor:entropyalgorithm}
Assume that there exists an algorithm $\mathcal{A}$ that given $n$ random variables $\{Z_i\},i\in[n]$ with distributions $\{p_i\},i\in[n]$  each with $n$ states, outputs the distribution of the random variable $E$ with minimum cardinality and functions $\{f_i,i\in[n]\}$ where $Z_i = f_i(E)$. 

Consider the causal pair $X\rightarrow Y$ where $Y = f(X,E_0)$. Assume that the cardinality of $E_0$ is less than $n(n-1)$, and $f$ is generic. Then, $\mathcal{A}$ can be used to identify the true causal direction with probability 1, if $X,E_0$ are selected uniformly randomly from the proper dimensional simplices.
\end{corollary}
\begin{proof}
Feed the set of conditional distributions $\{\pr{Y|X=i}:i\in[n]\}$ and $\{\pr{X|Y=i}:i\in[n]\}$ to  $\mathcal{A}$ to obtain $E$, $\tilde{E}$. From Theorem \ref{thm:main}, with probability 1, $\mathcal{A}$ identifies $\tilde{E}$ with $\card(\tilde{E})\geq n(n-1)$. Then since $\card(E)\leq \card(E_0)<\card(\tilde{E})$, comparing cardinalities give the true direction.
\end{proof}

Corollary \ref{cor:entropyalgorithm} gives an algorithm for finding the true causal direction: Estimate $E,\tilde{E}$ with minimum $H_0$ entropy and declare $X\rightarrow Y$ if $|\tilde{E}|>|E|$ and declare $X\leftarrow Y$ if $|\tilde{E}|<|E|$. The result easily extends to the case where $X$ and $Y$ are allowed to have different number of states:
\begin{proposition}[\bf Inference algorithm]
Suppose $X\rightarrow Y$. Let $X\in\mathcal{X},Y\in \mathcal{Y}$, $|\mathcal{X}|=n,|\mathcal{Y}|=m$. Assume that $\mathcal{A}$ is the algorithm that finds the exogenous variables $E,\tilde{E}$ with minimum cardinality. Then, if the underlying exogenous variable $E_0$ satisfies $|E_0|<n(m-1)$, with probability 1, we have $|X|+|E|<|Y|+|\tilde{E}|$. 
\end{proposition}
Proof follows from Corollary \ref{cor:entropyalgorithm}, and by extending the proof of Theorem \ref{thm:main} to different cardinalities for $X$, $Y$.

Unfortunately, it turns out there does not exist an efficient algorithm $\mathcal{A}$, unless $\textsc{P=NP}$:
\begin{theorem}
\label{thm:cardNPhard}
Given a conditional distribution matrix $\mat{Y|X}$, identifying $E\ci X$ with minimum support size such that there exist a function $f$ with $Y = f(X,E)$ is $\NP$ hard.
\end{theorem}

The hardness of this problem sets us to search for alternative approaches.

\section{Causal Model with Minimum \texorpdfstring{$H_1$}{} Entropy}
In this section, we propose a way to identify the causal model that explains the observational data with minimum Shannon entropy ( entropy in short ). Entropy of a causal model is measured by the number of random bits required to generate its input. In the causal graph $X\rightarrow Y$, where $Y = f(X,E)$, we identify the exogenous variable $E\ci X$ with minimum entropy. We show that this corresponds to a known problem which has been shown to be $\NP$ hard. Later we propose a greedy algorithm.

Notice that $H(E)$ is different from the conditional entropy $H(Y|X)$. 
Certainly, since $Y=f(X,E)$, $H(Y|X)\leq H(E)$. 
The key is that since $E$ is forced to be independent from $X$, $H(E)$ cannot be lowered to $H(Y|X)$.
To see this, we can write $H(Y|X) = \sum_i p_X(i)H(Y|X=i)$, whereas since conditional probability distribution of $Y|X=i$ is the same as the distribution of $f_i(E)$ for some function $f_i$, we have $H(E)\geq \max_i H(Y|X=i)$. 

\subsection{Finding {\itshape{\mdseries  E}} with minimum entropy}
Consider the equation $Y = f(X,E),X\ci E$. 
Let $f_x:\mathcal{E}\rightarrow \mathcal{Y} $ be the function mapping $E$ to $Y$ when $X=x$, i.e., $f_x(E) \coloneqq f(x,E)$. Then $\pr{Y=y|X=x} = \pr{f_x(E) = y|X = x} = \pr{f_x(E)=y}$. The last equality follows from the fact that $X\ci E$. Thus, we can treat the conditional distributions $\pr{Y|X=x}$ as distributions that emerge by applying some function $f_x$ to some unobserved variable $E$. Then the problem of identifying $E$ with minimum entropy given the joint distribution $p(x,y)$ becomes equivalent to, given distributions of the variables $f_i(E)$, finding the distribution with minimum entropy (distribution of $E$), such that there exists functions $f_i$ which map this distribution to the observed distributions of $Y|X=i$. It can be shown that $H(E)\geq H(f_1(E),f_2(E),\hdots,f_n(E))$. Regarding $f_i(E)$ as a random variable $U_i$, the best lower bound on $H(E)$ can be obtained by minimizing $H(U_1,U_2,\hdots,U_n)$. We can show that we can always construct an $E$ that acheives this minimum. Thus the problem of finding the exogenous variable $E$ with minimum entropy given the joint distribution $p(x,y)$ is equivalent to the problem of finding the minimum entropy joint distribution of the random variables $U_i = (Y|X = i)$, given the marginal distributions $p(Y|X=i)$:

\begin{theorem}[\textbf{Minimum Entropy Causal Model}]
\label{thm:equivalencetoentropy}
Assume that there exists an algorithm $\mathcal{A}$ that given $n$ random variables $\{Z_i\},i\in[n]$ with distributions $\{p_i\},i\in[n]$ each with $n$ states, outputs the joint distribution over $Z_i$ consistent with the given marginals, with minimum entropy.

Then, $\mathcal{A}$ can be used to find the causal model $\mathcal{M} =  (\{X,Y\},E,X\rightarrow Y, p_{X,E})$ with minimum input entropy, given any joint distribution $p_{X,Y}$. 
\end{theorem}

The problem of minimizing entropy subject to marginal constraints is non-convex. In fact, it is shown in \cite{Kovacevic2012} that minimizing the joint entropy of a set of variables given their marginals is $\NP$ hard. Thus we have the following corollary:

\begin{corollary}
\label{cor:entropyNPhard}
Finding the causal model $\mathcal{M} = (\{X,Y\}, E, f, X\rightarrow Y, p_{E,X})$ with minimum $H(E)$ that induce a given distribution $p(x,y)$ is $\NP$ hard.
\end{corollary}

For this, we propose a greedy algorithm. Using entropy to identify $E$ instead of cardinality, despite both turning out to be NP hard, is useful since entropy is more robust to noise in data. In real data, we estimate the probability values from samples, and noise is unavoidable.

\subsection{A Conjecture on Identifiability with \texorpdfstring{$H_1$}{} Entropy}
We have the following conjecture, supported by artificial and real data experiments in Section \ref{sec:simulations}.

\begin{conjecture}
\label{conj:entropy}
Consider the causal model $\mathcal{M}=(\{X,Y\},E,f,X\rightarrow Y,p_{X,E})$ where discrete random variables $X,Y$ have $n$ states, $E\ci X$ has $\theta$ states. 

If the distribution of $X$ is uniformly randomly selected from the $n-1$ dimensional simplex and distribution of $E$ is uniformly selected from the probability distributions that satisfy  $H_1(E)\leq \log{n}+\mathcal{O}(1)$ and $f$ is randomly selected from all functions $f:[n]\times [\theta]\rightarrow [n]$, then with high probability, any $\tilde{E}\ci Y$ that satisfies $X = g(Y,\tilde{E})$ for some deterministic $g$ entails $H(X)+H(E)<H(Y)+H(\tilde{E})$.
\end{conjecture}

\begin{proposition}[Assuming Conjecture \ref{conj:entropy}]
\label{prop:entropybasedalgorithm}
Assume there exists an algorithm $\mathcal{A}$ that given $n$ random variables $\{Z_i\},i\in[n]$ with distributions $\{p_i\},i\in[n]$ each with $n$ states, outputs the distribution of the random variable $E$ with minimum entropy and functions $\{f_i\},i\in[n]$ where $Z_i = f_i(E)$. 

Consider the causal pair $X\rightarrow Y$ where $Y = f(X,E_0)$, and cardinality of $E_0$ is $cn$ for some constant $c$, and $f$ is selected randomly. Then, $\mathcal{A}$ can be used to identify the true causal direction with high probability, if $X,E_0$ are uniformly random samples from the proper dimensional simplices.
\end{proposition}

\subsection{Greedy Entropy Minimization Algorithm}
Given $m$ discrete random variables with $n$ states, we provide a heuristic algorithm to minimize their joint entropy given their marginal distributions. The main idea is the following: Each marginal probability constraint must be satisfied. For example, for the case of two variables with distributions $p_1,p_2$, $i$th row of joint distribution matrix should sum to $p_1(i)$. The contribution of a probability mass to the joint entropy only increases when probability mass is divided into smaller chunks: $-p_1(i)\log{p_1(i)}\leq -a\log{a}-b\log{b}$, when $p_1(i) = a+b,$ for $a,b\geq 0$. Thus, we try to keep large probability masses intact to assure that their contribution to the joint distribution is minimized.
 
 \begin{figure*}[t]
\begin{subfigure}[b]{0.33\linewidth}
\includegraphics[width=\linewidth]{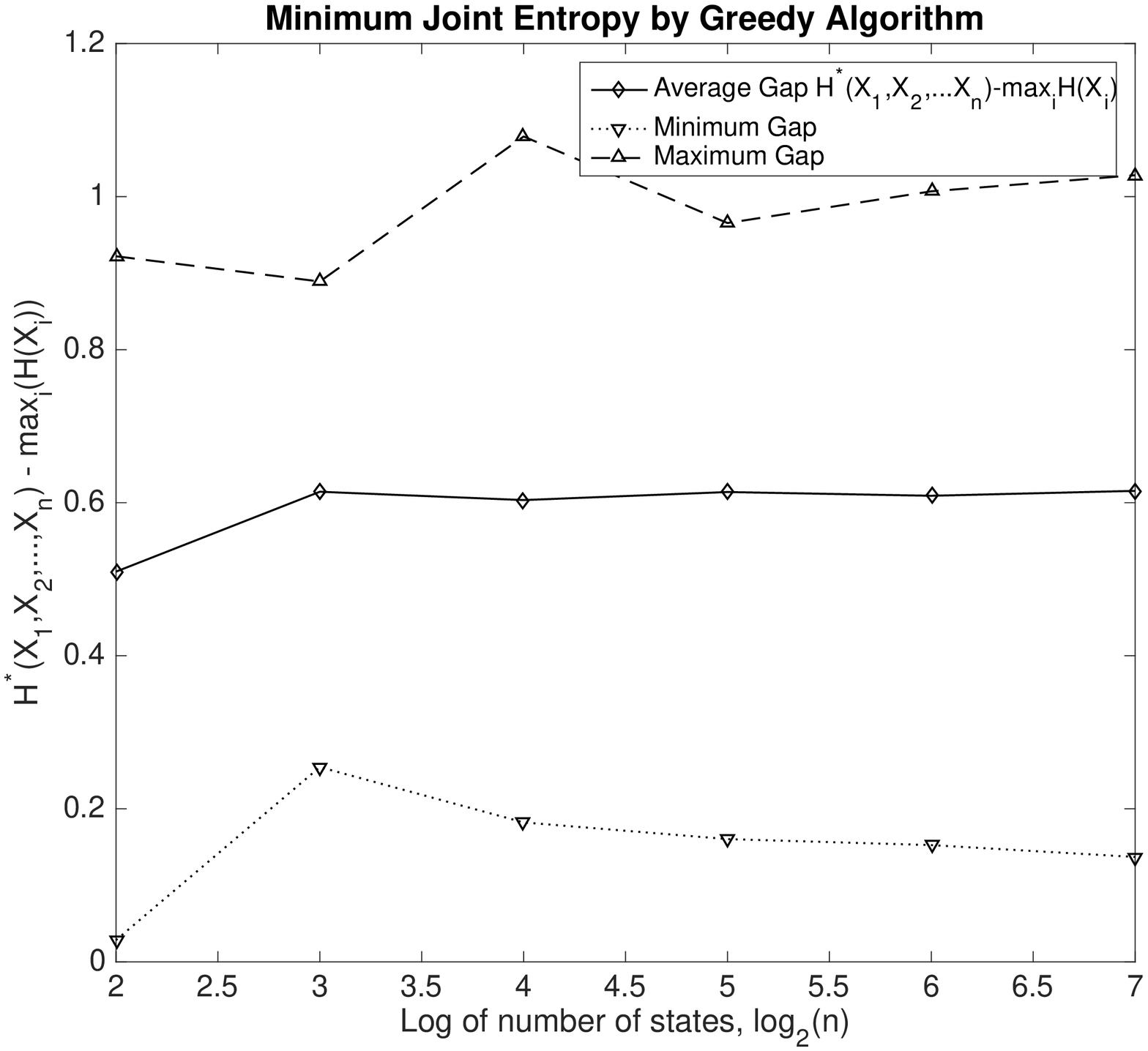}
\caption{Greedy Joint Entropy Min.}
\label{fig:greedyEntropyPerformance}
\end{subfigure}
\begin{subfigure}[b]{0.33 \linewidth}
\includegraphics[width=\linewidth]{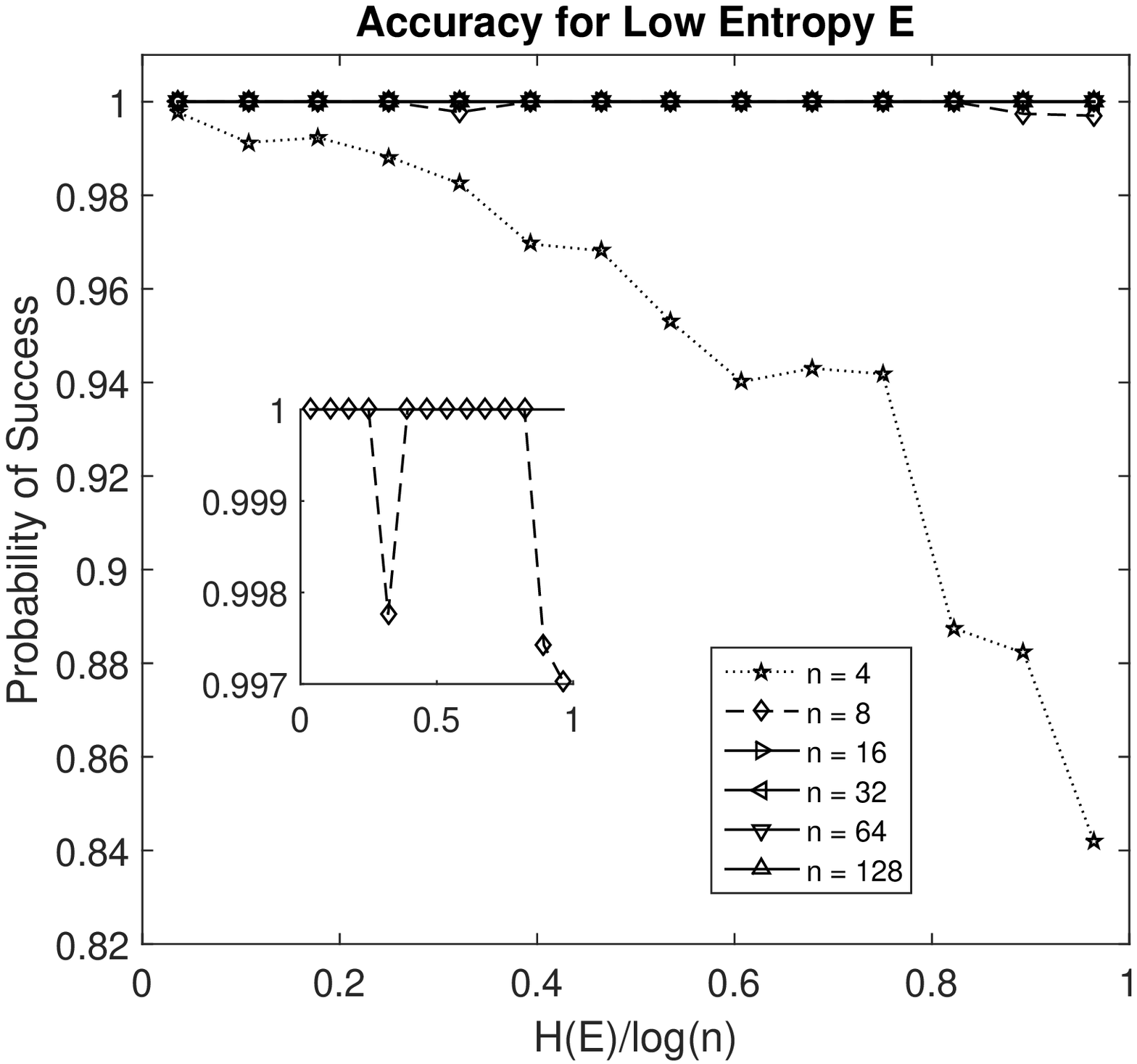}
\caption{Identifiability, Artificial Data}
\label{fig:empiricalidentifiability}
\end{subfigure}
\begin{subfigure}[b]{0.33 \linewidth}
\includegraphics[width=\linewidth]{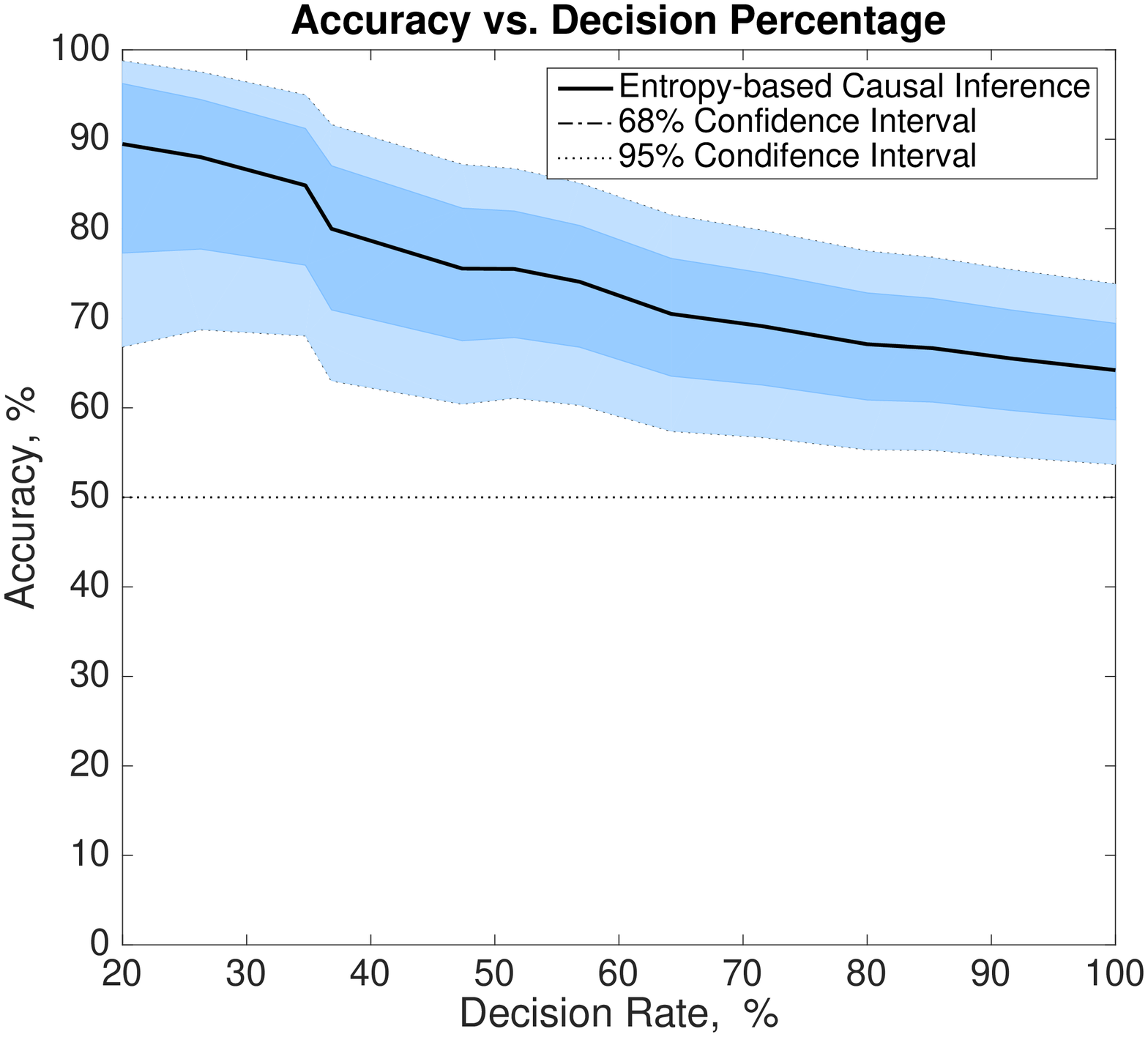}
\caption{Performance on Real Data }
\label{fig:realdata_accvsdecrate}
\end{subfigure}
\caption{\small (a) Performance of greedy joint entropy minimization algorithm: $n$ distributions each with $n$ states are randomly generated for each value of $n$. As can be seen, the minimum joint entropy obtained by the greedy algorithm is at most 1 bit away from the largest marginal $\max_iH(X_i)$. (b) Identifiability with Entropy: We generate distributions of $X,Y$ by randomly selecting $f,X,E$. Probability of success is the fraction of points where $H(X,E)<H(Y,\tilde{E})$. As observed, larger $n$ drives probability of success to 1 when $H(E)\leq \log{n}$, supporting Conjecture \ref{conj:entropy}. (c) Real Data Performance: Decision rate is the fraction of samples for which algorithm makes a decision for a causal direction. A decision is made when $|H(X,E)-H(Y,\tilde{E})| > t\log_2{n}$, where $t$ determines the decision rate. Confidence intervals are also provided.}
\end{figure*}
	
We propose Algorithm \ref{alg:heuristic}. The sorting step is only to simplify the presentation. Hence, although the given algorithm runs in time 
$\mathcal{O}(m^2n^2\log{n})$, it can easily be reduced to 
$\mathcal{O}(\max(mn\log{n},m^2n))$ by dropping the sorting step. 
The algorithm simply proceeds by removing the most probability mass it can at each round. This makes sure the large probability masses remain intact.
\begin{algorithm}[ht!]
\begin{small}
    \caption{Joint Entropy Minimization Algorithm}
   \label{alg:heuristic}
\begin{algorithmic}[1]
    \State {\bfseries Input:} Marginal distributions of $m$ variables each with $n$ states, in matrix form $\mat{M} = [p_1^T;p_2^T;...,p_m^T]$.
    \State $e = [ \hspace{0.1in}]$
    \State Sort each row of $M$ in decreasing order.
    \State Find minimum of maximum of each row: $r\leftarrow \min_i(p_i(1))$
     \While  {$r>0$} 
     \State $e \leftarrow [ e , r ]$
     \State Update maximum of each row: $p_i(1)\leftarrow p_i(1)-r, \forall i$
     \State Sort each row of $M$ in decreasing order.
     \State $r\leftarrow \min_i(p_i(1))$
    \EndWhile
    \State \Return $e$.
\end{algorithmic}
\end{small}
\end{algorithm}

One can easily construct the joint distribution using a variant: Instead of sorting, at each step, find $r = \min_i\{\max_j\{p_i(j)\}\}$ and assign $r$ to the element with coordinates $(a_i)$, where $a_i=\argmax_j{p_i(j)}$.
\begin{lemma}
\label{lem:greedyUpperBound}
Greedy entropy minimization outputs a point with entropy at most $\log{m}+\log{n}$.
\end{lemma}

Lemma \ref{lem:greedyUpperBound} follows from the fact that the algorithm returns a support of size at most $m(n-1)+1$.

We also prove that, when there are two variables with $n$ dimensions, the algorithm returns a point that satisfies the KKT conditions of the optimization problem, which implies that it is a local optimum (see Proposition \ref{prop:greedyLocalOptima} in the Appendix).

\section{Experiments}
\label{sec:simulations}
In this section, we test the performance of our algorithms on real and artificial data. First, we test the greedy entropy minimization algorithm and show that it performs close to the trivial lower bound. Then, we test our conjecture of identifiability using entropy. Lastly, we test our entropy-minimization based causal identification technique on real data. 

In order to test our algorithms, we sample points in proper dimensional simplices, which correspond to distributions for $X$ and $E$. Distribution of points are uniform for selecting the distribution of $X$. It is well-known that a vector $[x_i/Z]_i$ is uniformly randomly distributed over the simplex, if $x_i$ are i.i.d. exponential random variables with parameter 1, and $Z=\sum_ix_i$ \cite{Onn2011}. To sample low-entropy distributions for $E$, instead of exponential, we use a heavy tailed distribution for sampling each coordinate. Specifically, we use $[e_i/Z]_i$, where $e_i$ are i.i.d. log-normal random variables with parameter $\sigma$. We observe that this allows us to sample a variety of distributions with small entropy.

\textbf{Performance of Greedy Entropy Minimization}:
We sample distributions for $n$ random variables $\{X_i\},i\in [n]$ each with $n$ states and apply Algorithm \ref{alg:heuristic} to minimize their joint entropy. We compare our greedy joint entropy minimization algorithm with the simple lower bound of $\max_i{H(X_i)}$. Figure \ref{fig:greedyEntropyPerformance} shows average, maximum and minimum excess bits relative to this lower bound. Contrary to the pessimistic bound of $\log{n}$ bits, joint entropy is at most 1 bit away from $\max_i H(X_i)$ for the given range of $n$.

\textbf{Verifying Entropy-Based Identifiability Conjecture}:
In this section, we empirically verify Conjecture \ref{conj:entropy}. The distributions for $X$ are uniformly randomly sampled from the simplex in $n$ dimensions. We also select $f$ randomly (see implementation details). For the log-normal parameter $\sigma$ used for sampling the distribution of $E$ from the $n(n-1)$ dimensional simplex, we sweep the integer values from 2 to 8. This allows us to get distribution samples from different regimes. We only consider the samples which satisfy $H(E)\leq \log{n}$.

After sampling $E,X,f$, we identify the corresponding $\mat{Y|X}$ and $\mat{X|Y}$ for $Y=f(X,E)$. We apply greedy entropy minimization on the columns of the induced distributions $\mat{Y|X},\mat{X|Y}$ to get the estimates $E,\tilde{E}$ for both causal models $Y=f(X,E)$ and $X=g(Y,\tilde{E})$, respectively. Figure \ref{fig:empiricalidentifiability} shows the variation of success probability, i.e., the fraction of samples which satisfy $H(X)+H(E)<H(Y)+H(\tilde{E})$. As observed, as $n$ is increased, probability of success converges to 1, when $H(E)\leq \log{n}$, which supports the conjecture.

\textbf{Experiments on Real Cause Effect Pairs}:
We test our entropy-based causal inference algorithm on the CauseEffectPairs repository \cite{CauseEffectRepo}. ANM have been reported to achieve an accuracy of $63\%$ with a confidence interval of $\pm 10\%$ \cite{Mooij2016}. We also use the binomial confidence intervals as in \cite{Clopper1934}. 

The cause effect pairs show very different characteristics. From the scatter plots, one can observe that they can be a mix of continuous and discrete variables. The challenge in applying our framework on this dataset is choosing the correct quantization. Small number of quantization levels may result in loss of information regarding the joint distribution, and a very large number of states might be computationally hard to work with. We pick the same number of states for both $X$ and $Y$, and use a uniform quantization that assures each state of the variables has $\geq 10$ samples on average. From the samples, we estimate the conditonal transition matrices $\mat{Y|X}$ and $\mat{X|Y}$ and feed the columns to the greedy entropy minimization algorithm (Algorithm \ref{alg:heuristic}), which outputs an approximate of the smallest entropy exogenous variable. Later we compare $H(X,E)$ and $H(Y,\tilde{E})$ and declare the model with smallest input entropy to be the true model, based on Conjecture \ref{conj:entropy}.

For a causal pair, we invoke the algorithm if $|H(X,E)-H(Y,\tilde{E})| \geq t\log(n)$ for threshold parameter $t$, which determines the decision rate. Accuracy becomes unstable for very small decision rates, since the number of evaluated pairs becomes too small. At $100\%$ decision rate, algorithm achieves $64.21\%$ which is slightly better than the $63\%$ performance of ANM as reported in \cite{Mooij2016}. In addition, our algorithm only uses probability values, and is applicable to categorical as well as ordinal variables.

\section*{Acknowledgements}
This research has been supported by NSF Grants CCF 1344179, 1344364, 1407278, 1422549 and ARO YIP W911NF-14-1-0258.

{\small
\bibliographystyle{plain}
\bibliography{causalinferenceSimple.bib}
}
\newpage
\section{Appendix}
\subsection{Proof of Lemma \ref{lem:characterization}}
($\Rightarrow$)Assume there exists a causal model $\mathcal{M} = (\{X,Y\},E, f, X\rightarrow Y, p_{X,E})$. Without loss of generality, assume $\mathcal{E}=[m]$ and $p_E = [e_1, e_2, \hdots, e_m ]$. We have
\begin{equation}
\label{eq:decomposition}
p(y|x) = \sum_{e\in\mathcal{E}}p(y|x,e)p(e|x) = \sum_{e\in\mathcal{E}}p(y|x,e)p(e)
\end{equation}
since $E\ci X$. Define the matrix $\mat{Y|X_k}(i,j) \coloneqq \pr{Y=i|X=j,E=k}$. Then, from (\ref{eq:decomposition}), we can decompose the conditional probability distribution matrix $\mat{Y|X}$ as follows:
\begin{equation}
\mat{Y|X} = \sum_{k\in\mathcal{E}}e_k\mat{Y|X_k}.
\end{equation}

Since $f$ is deterministic, each value of $X$ is mapped to exactly one value of $Y$, when $E$ is conditioned on. Thus each column of $\mat{Y|X_k}$ has exactly a single 1 with remaining entries being zeros. Thus, each entry of $\mat{Y|X}$ is a subset sum of $p_E$. Let $S_{i,j}$ represent this subset, i.e., $p(y=i|x=j) = \sum_{k\in S_{i,j}}e_k$. Notice that the $x$th column of $\mat{Y|X}$ is the conditional distribution $\pr{Y|X=x} = \pr{f(x,E)|X=x}=\pr{f(x,E)} = \pr{f_x(E)}$, where $ f_x(E) \coloneqq f(x,E)$. Since $f_x$ is a deterministic function, each value in its domain maps to exactly one value in its range. This implies that $S_{y,x} = f_x^{-1}(y)$ are disjoint for fixed $x$, i.e., $S_{i,j}\neq S_{l,j}\forall l\neq i$. Also, since each value of $E$ must be mapped to a value of $Y$ by $f_x$, the union of $S_{i,j}$ over $i$ must be the whole support $[m]$.

Define $\mat{m}_{i,j}$ to be the length $m$ vector which is 1 in the columns indexed by $S_{i,j}$. Construct $\mat{M}$ from the rows $\mat{m}_{i,j}$ such that $m_{i,j}$ is the $i+(j-1)n$th row of $\mat{M}$. By construction, $\mat{M}$ is a block partition matrix. Pick $\mat{e} = [e_1,e_2,\hdots, e_m]$. Then we have $vec(\mat{Y|X}) = \mat{Me}$.

($\Leftarrow$) For reverse direction, assume there exists matrices $\mat{M,e}$ with block partition $\mat{M}$ and $\sum_i e_i = 1$, such that $vec(\mat{Y|X}) = \mat{Me}$. Define $E$ to be the random variable independent from $X$, with probability distribution $p_E\coloneqq \mat{e}$ and support $\mathcal{E} = [m]$. 

Let $\mat{w}_i$ be the $i$th column of $\mat{M}$. Then we have $vec(\mat{Y|X}) = \sum_i e_i \mat{w}_i$. Now de-vectorize $\mat{Y|X}$ and $\mat{w}_i$ to have
\begin{equation}
\mat{Y|X} = \sum_{i=1}^m e_i \mat{U}_i,
\end{equation}
where $\mat{w}_i = vec(\mat{U}_i)$. Each column of $\mat{U}_i$, comes from distinct size-$n$ blocks of $\mat{M}$ and since $\mat{M}$ is block partition, each column of $\mat{U}_i$ contains a single 1 $\forall i$. Thus each $\mat{U}_i$ represent a valid map $f_i:\mathcal{X}\rightarrow \mathcal{Y}$, where $f_i(x)$ is given by the nonzero row of $x$th column of $\mat{U}_i$.

A function $f$ with two input variables $X,E$, where $E\in\mathcal{E} = [m]$ is completely determined by the set of functions $\{f(X,1),f(X,2),...,f(X,m)\}$. Let $f(X,e)$ be the function described by the matrix $\mat{U}_e$ and $f$ be the function determined by $\{f(X,e),e\in\mathcal{E}\}$. Apply the constructed $f$ on $X,E$ to get $Z = f(X,E)$. Then we have $\pr{Y = y|X=x} = \pr{Z = y|X=x}\forall x$, since both $Z$ and $Y$ induce the same conditional distribution $\mat{Y|X}$. 

For any set of realizations of $(X,Y) = \{x_i,y_i\}$, we can construct a set of realizations of $E, \{e_i\}$ based on the conditional distribution $\pr{E|Z=y_i,X=x_i}$. Then we have $y_i = f(x_i,e_i)$. Also, since $\pr{Z=y,X=x} = \pr{Y=y,X=x}$ this process induces the same joint distributions between variables $X,Y,E$ and $X,Z,E$, i.e.,  $\pr{X=x,E=e,Z=y} = \pr{X=x,E=e,Y=y}$, and conditional independence statement implied by one holds for the other. Thus $X\ci E$.

\subsection{Unidentifiability without assumptions}
Here we prove that we can fit causal models in both directions $X\rightarrow Y,Y\rightarrow X$ given any joint distribution.
\begin{lemma}
\label{lem:unidentifiability}
Let $X\in \mathcal{X},Y\in\mathcal{Y}$ be discrete random variables with an arbitrary joint distribution $p(x,y)$, where $\lvert\mathcal{X}\rvert = m,\lvert\mathcal{Y}\rvert = n$. Then there exists two causal models $\mathcal{M}_1=(\{X,Y\},E,f,X\rightarrow Y,p_{X,E})$ and  $\mathcal{M}_2=(\{X,Y\},\tilde{E},g,X\leftarrow Y,p_{Y,\tilde{E}})$  with $E\ci X$ and $\tilde{E}\ci Y$ that induce the same joint distribution $p(x,y)$.
\end{lemma}

We will prove by construction. Consider the conditional probability transition matrix $\mat{Y|X}$. Without loss of generality, assume $\mathcal{X} = [m],\mathcal{Y}=[n]$. From Lemma \ref{lem:characterization}, it is sufficient to show that there exists $\mat{M}\in\mathcal{C}$ and $\mat{e}$ such that $vec(\mat{Y|X}) = \mat{Me}$. 

For now, assume that each entry of $\mat{Y|X}$ is a rational number. Scale the fractional form of each term in $\mat{Y|X}$ so that each denominator becomes the same as the least common multiple of denominators. Denote this least common multiple by $\lambda$. Let $\epsilon$ be $1/\lambda$.

Let $ \mat{G}_0 = \mat{Y|X}$ and apply the following procedure for $i$ from $1$ to $\lambda$ in order to construct $\{F_i,i\in[\lambda]\}$: Set $\mat{G}_{i+1}\text{ to }  \mat{G}_i - \epsilon F_i$, where $F_i$ is the $n\times m, \{0,1\}$ matrix containing only a single 1 in the largest entry of every column of current $\mat{G}_i$. This procedure is called to be successful if $\mat{G}_i$ is set to all zero matrix for $i=\lambda$.

The above procedure iteratively removes 1 from the numerator of the fractional form of one probability value per column (of $\mat{Y|X}$). Since each column sums to 1, numerators of each column sum to $\lambda$. Thus the procedure is successful. Thus we have,
\begin{equation}
\mat{Y|X} = \sum_{k=1}^\lambda \epsilon F_k.
\end{equation}

Let $\mat{e} = [\epsilon, \epsilon, ..., \epsilon]$ be a length-$\lambda$ vector. Each entry of $\mat{Y|X}$ is a subset sum of $\mat{e}$. Also, since $F_i$ contains a single 1 per column by construction, every subset of $\mat{e}$ is disjoint within a column. Thus, we can construct $\mat{M}\in\mathcal{C}$ such that $vec(\mat{Y|X}) = \mat{Me}$.

If the entries are not fractional, we can still find small enough $\epsilon$ to complete the above procedure.

The same process can be implemented with $\mat{X|Y}$ to obtain $\tilde{E},g$ such that $X = Y(g,\tilde{E})$.

\subsection{Proof of Lemma \ref{lem:decompUpper}}
We prove by construction for any given joint probability distribution. Consider the decomposition described in the proof of Lemma \ref{lem:unidentifiability}. Assume without loss of generality that $\mathcal{X} = \mathcal{Y} = [n]$. Now, instead of taking out $\lambda=1/\epsilon$, at step $i$, remove minimum of the maximum probability values at each column of the current $\mat{G}_i$ matrix from the maximum probability locations. Thus, at each iteration $i$, at least one entry of the matrix $\mat{G}_i$ is zeroed out. Since sum of the values in each column remains the same after each iteration, after at most $n(n-1)$ steps, each column must have the same single nonzero value.  Thus the algorithm finalizes in $n(n-1)+1$ steps.

Notice that this algorithm is the same as the entropy minimization algorithm we propose in Algorithm \ref{alg:heuristic}. For a more detailed explanation of the algorithm steps, see Algorithm \ref{alg:heuristic}.

In the case when the matrix has zero entries, it can be shown that one can always find a decomposition with $nnz-1$ terms, where $nnz$ is the number of non-zero elements in the matrix.

\subsection{Proposition \ref{prop:helper}}
\begin{proposition}
\label{prop:helper}
Let the columns of $\mat{Y|X}$ be $n$ points independently sampled from the uniform distribution over the $n-1$ simplex. Then, with probability 1, $\nexists (\mat{M},\mat{e})$ with  $vec(\mat{Y|X}) = \mat{M}\mat{e}$ for $\mat{M}\in\{0,1\}^{n^2 \times m}$, when  $m< n(n-1)$.
\end{proposition}

\begin{proof}
We use the following technique to generate uniformly randomly sampled points on the simplex in $n$ dimensions \cite{Onn2011}:
\begin{lemma}
\label{lem:uniformSimplex}
Let $x_i\sim \mathcal{U}[0,1]$ for $i\in[n-1]$ be i.i.d random variables, ordered such that $x_i\geq x_j$ for $i>j$. Let $x_0 = 0$ and $x_n = 1$. Then $\mat{u} = [u_i]_{i\in[n]}$ where $u_i = x_i-x_{i-1}, \forall i\in[n]$, is a random vector uniformly distributed over $n-1$ simplex.
\end{lemma}

First, we construct $vec(\mat{Y|X})$ directly using $n(n-1)$ uniform i.i.d. random variables: Consider $\tilde{x}_i$ for $i\in[n^2]$ where each consecutive block $\{\tilde{x}_{jn+1},\tilde{x}_{jn+2},\hdots,\tilde{x}_{jn+n}\},j\in \{0,1,\hdots,n-1\}$ of size $n$ is sampled from the generative model in Lemma \ref{lem:uniformSimplex} before reordering. Thus $\tilde{x}_i$ are i.i.d. with $\tilde{x}_i\sim \mathcal{U}[0,1]$ for $n(n-1)$ indexes by construction. Order $\tilde{x}_i$ within each block in ascending order to get $x_i$ in accordance with Lemma \ref{lem:uniformSimplex}. Defining the vectors $\mat{x} = [x_i]$ and $\mat{\tilde{x}} = [\tilde{x}_i]$, we have $\mat{x} = \mat{P\tilde{x}}$ for some permutation matrix $\mat{P}$. From $\mat{x}$, we can construct $\mat{\tilde{z}} = vec(\mat{Y|X})$ using the same map used in Lemma \ref{lem:uniformSimplex} for each block (the map from $u_i$ from $x_i$ in Lemma \ref{lem:uniformSimplex}). This construction is a linear map $\mat{H}$ where each submatrix of $\mat{H}$ is full rank. 

For the sake of contradiction, let $\mat{\tilde{z}} = \mat{\tilde{M}e}$ be a decomposition where $\mat{\tilde{M}}\in\{0,1\}^{n(n-1)\times m}$ where $m<n(n-1)$. Ignoring the entries where $\tilde{z_i} = 1-x_{i-1}$, we can relabel $\mat{\tilde{z}}$ to get $\mat{z}$ with $\mat{z}\in[0,1]^{n(n-1)}$. Correspondingly, we can write $\mat{z} = \mat{Me}$, where $\mat{M}$ is the submatrix of $\mat{\tilde{M}}$ obtained by ignoring the corresponding rows.

The construction of $\mat{z}$ from $\mat{x}$ based on the above construction yields $\mat{z} = \mat{Wx}$ for a full rank matrix $\mat{W}$. Notice that any subset of rows are linearly independent due to specific structure of $\mat{H}$.

Let $r$ be the rank of $\mat{M}$. Clearly $r<n(n-1)$. Then, some of the (at least $n(n-1)-r$) rows of $\mat{M}$ can be written as a unique linear combination of $r$ linearly independent rows. 

Consider one such row $m_0$, where $m_0 = \sum_{i=1}^r\alpha_im_i$ and $m_i$ are a set of linearly independent rows of $\mat{M}$. Define $\mat{a} = [-1, \alpha_1, \alpha_2, ... \alpha_r]^T$. Take $r+1$ rows of $\mat{W}$ corresponding to the selected $z_i$'s to form $\mat{W}_r$. Then we have, $\mat{a}^T\mat{W}_r\mat{x}=\mat{a}^T\mat{W}_r\mat{P}\mat{\tilde{x}} = 0$. Recall that any subset of rows of $\mat{W}$ is full rank, and $\mat{P}$ is a permutation matrix. Hence, left nullspace of $\mat{W}_r\mat{P}$ is empty, implying that $\mat{\tilde{x}}$ has to be orthogonal to the vector $\mat{a}^T\mat{W_r}\mat{P}\neq 0$. This is a probability 0 event for any nonzero $\mat{a}$, since each $\tilde{x_i}$ is independently sampled from a continuous distribution. Probability that all such constraints are satisfied is zero since it is less than the probability that one particular constraints is satisfied. This argument holds for any fixed $\mat{M}$. Since there are finitely many such $\{0,1\}$ matrices, probability that there exists such an $\mat{M}$ is 0.

Thus, unless $\mat{M}$ has rank at least $n(n-1)$, there does not exist a decomposition $\mat{M}\mat{e} = \mat{z}$ with probability 1.
\end{proof}

\subsection{Proof of Theorem \ref{thm:main}}
For sampling from the simplex, we use the following model:

\begin{lemma}[\cite{Onn2011}]
Let $x_i$ for $i\in [n]$ be independent, exponential random variables with mean 1. Then the random vector $\left[\frac{x_1}{\sum_i x_i}, \frac{x_2}{\sum_i x_i}, ..., \frac{x_n}{\sum_i x_i}\right]$ is uniformly distributed over the $n-1$  simplex.
\end{lemma}

Assume above generative model for sampling the distributions of $X$ and $E$: Let $\{x_i,i\in [n]\}$ and $\{e_i,i\in[\theta]\}$ be sets of independent identically distributed exponential random variables with mean 1. Assign $\pr{X=i} = \frac{x_i}{\sum_j x_j}$, and $\pr{E=k} = \frac{e_k}{\sum_j e_j}$.

Let $\mat{p} = \mat{vec(Y|X)}$. Then, as shown in Section \ref{subsec:identifiability}, we can write
\begin{equation}
\mat{p} = \mat{M}\mat{e}
\end{equation}

Notice that $j$th block of $n$ rows of $\mat{p}$ give the conditional probability distribution of $Y$ given $X=j$. Let $S_{i,j}$ represent the set of indices of $\mat{e}$ that contribute to the probability of observing $Y=i$ given $X=j$, i.e., $p_{i,j} = \frac{1}{\sum_je_j}\sum\limits_{k\in S_{i,j}}e_k$. Thus $i$th row in $j$th block, or equivalently $i + (j-1)n$th row of $\mat{M}$ is 1 in the indices $S_{i,j}$.

The fact that $f$ is generic implies each row of $\mat{M}$ is distinct and non-empty.

For the sake of contradiction, assume that there exists $\tilde{E}\ci Y$ with cardinality $m<n(n-1)$, with some deterministic function $g$ such that $X=g(Y,\tilde{E})$ induces the same joint distribution $p(x,y)$. By Lemma \ref{lem:characterization}, this implies that there exists $\mat{\tilde{M}}\in \{0,1\}^{n^2\times m}$ end $\mat{\tilde{e}}$ such that 
\begin{equation}
\label{eq:reverseAppendix}
\mat{q} = \mat{\tilde{M}}\mat{\tilde{e}},
\end{equation}
where $\mat{q} = \mat{vec(X|Y)}$ and $\mat{\tilde{e}}$ is the probability distribution vector of $\tilde{E}$.

Let $q_{i,j} = \pr{X=i|Y=j}$. Then from Bayes' rule, we have 
\begin{equation}
q_{i,j} = \frac{p_{j,i}x_i}{\sum_i p_{j,i}x_i} = \frac{x_i\sum_{k\in S_{j,i}}e_k}{\sum_i x_i\sum_{k\in S_{j,i}}e_k}.
\end{equation}
Notice that the denominators $\sum_j x_j$ and $\sum_j e_j$ in each term disappear due to cancellation of numerator and denominator. 

Thus we have 
\begin{align}
\mat{q} = 
&\left[ \frac{p_{11}x_1}{\sum_k p_{1k}x_k} ,  \frac{p_{12}x_2}{\sum_kp_{1k}x_k} , \cdots ,\right.\\
&\left.  \frac{p_{1n}x_n}{\sum_kp_{1k}x_k} , \frac{p_{21}x_1}{\sum_kp_{2k}x_k} , \cdots  \frac{p_{nn}x_n}{\sum_kp_{nk}x_k}
\right]^T
\end{align}

From (\ref{eq:reverseAppendix}), drop the rows of $\mat{q}$ that contain $x_n$ in the numerator as well as the corresponding rows of $\mat{\tilde{M}}$. The new linear system becomes:
\begin{equation}
\mat{\bar{q}} = \mat{\bar{M}}\mat{\tilde{e}},
\end{equation}
where $\mat{\bar{q}}$ and $\mat{\bar{M}}$ are the desribed submatrices of $\mat{q}$ and $\mat{\tilde{M}}$ respectively.

$\mat{\bar{M}}$ has $n(n-1)$ rows and $m$ columns. We have 
\begin{equation}
\text{rank}(\mat{\bar{M}})\leq m<n(n-1).
\end{equation}
Since rank of $\mat{\bar{M}}$ is less than $n(n-1)$, the rows of $\mat{\bar{M}}$ are linearly dependent. This implies there is at least one set of coefficients $\{\alpha_i\}$ not identically zero that satisfies $\mat{\alpha\mat{\bar{M}}} = 0$, where 
\begin{equation}
\mat{\alpha} = \begin{bmatrix}
\alpha_{1,1}, & \alpha_{1,2}, & \cdots ,&\alpha_{1,n-1},& \alpha_{2,1},& \cdots , & \alpha_{n,n-1}
\end{bmatrix}.
\end{equation}

Then,
\begin{equation}
\mat{\alpha}\mat{\bar{q}} = \mat{\alpha \bar{M}\tilde{e}} =0 
\end{equation}

Hence, the elements of $\mat{\bar{q}}$ should satisfy the linear equation. Then this linear equation in terms of $q_{i,j}$ can be written as:

\begin{alignat}{2}
&\sum_{i=1}^n  \frac{\sum_{j=1}^{n-1}\alpha_{i,j}p_{i,j}x_j}{\sum_{j=1}^{n}p_{i,j}x_j} = \frac{\sum_{j=1}^{n-1}\alpha_{1,j}p_{1,j}x_j}{\sum_{j=1}^{n}p_{1,j}x_j} +\frac{\sum_{j=1}^{n-1}\alpha_{2,j}p_{2,j}x_j}{\sum_{j=1}^{n}p_{2,j}x_j} \nonumber\\
& +\frac{\sum_{j=1}^{n-1}\alpha_{3,j}p_{3,j}x_j}{\sum_{j=1}^{n}p_{3,j}x_j} + \hdots + \frac{\sum_{j=1}^{n-1}\alpha_{n,j}p_{n,j}x_j}{\sum_{j=1}^{n}p_{n,j}x_j}= 0  \label{eq:polynomial}
\end{alignat}

Slightly abusing the notation, relabel $p_{i,j}$ as $p_{i,j} = \sum_{k\in S_{i,j}}e_k$, since $\sum_k e_k$ terms cancel in the expression above. We know from Section \ref{subsec:identifiability} that $S_{i,j}\cap S_{k,j} = \emptyset$ for $k\neq i$. Additionally, due to the assumption that $f$ is generic, we have $S_{i,j}\neq S_{i,k}$.

To prove contradiction, in the following we show that for any given non-zero $\alpha$, this equation is non-zero with probability 1.

To show this, we show that after equating the denominators, each term brings a unique monomial. Hence, the result is a polynomial where each $\alpha_{i,j}$ is accompanied with at least one unique monomial. Thus, the polynomial cannot be identically zero, and the probability of choosing a root of this polynomial is zero.

For now, assume that the denominator is finite. Later, we will show denominator is almost surely finite to complete the argument. 

\begin{lemma}
\label{lem:poly}
If $x_i$ and $e_i$ are i.i.d. exponential random variables with mean 1, and the function $f$ is generic, (\ref{eq:polynomial}) holds with probability 0, i.e., $\pr{(\ref{eq:polynomial})\text{ holds }} = 0$.
\end{lemma}
\begin{proof}

Multiply each term with the denominators of others to get a single fraction. Then, with a finite denominator, numerator should be zero for equation to hold. Define $c_1$ to be the coefficient of $x_n^{n-1}x_1$ in the numerator after equating the denominators. Since $x_n$ only appears due to terms from the denominator, we have
\begin{align}
c_1 &= \alpha_{1,1}(p_{1,1}p_{2,n}p_{3,n}\hdots p_{n,n})\nonumber\\
&+\alpha_{2,1}(p_{1,n}p_{2,1}p_{3,n}\hdots p_{n,n}) \nonumber\\
&+\cdots+\alpha_{n,1}(p_{1,n}p_{2,n}p_{3,n}\hdots p_{n,1})\label{eq:coeffofxn}
\end{align}
Since $S_{1,1}\neq S_{1,n}$ due to $f$ being generic, we have,
\begin{itemize}
\item[(a)] Either $\exists e_i\in S_{1,1},e_i\notin S_{1,n}$
\item[(b)] Or  $\exists e_i\notin S_{1,1},e_i\in S_{1,n}$
\end{itemize}

Without loss of generality, assume some $e_{i_1}\in S_{i,1}$, after a potential relabeling. This is possible since $S_{i,1}$ are disjoint for different $i$ and non-zero. (Then, as we will see $e_{i_1}^2$ only appears together with $\alpha_{i,1}$ in (\ref{eq:coeffofxn})). Similarly, assume some $e_{i_n}\in S_{i,n}$.

Consider the mulitiplier of coefficient $\alpha_{i,1}$. Assume case (a) holds for $S_{i,1}$, i.e., $\exists e_i\in S_{i,1},e_i\notin S_{i,n}$. Then, $\alpha_{i,1}$ is accompanied with term $e_i^2$ since $e_i\in S_{j,n}$ for some $j\neq i$. Also, it is easy to see that no other term contains $e_i^2$ since $S_{i,1}$ does not appear again and no $S_{j,1},j\neq i$ contains $e_i$.

Assume case (b) holds for the multiplier of $\alpha_{i,1}$, i.e., $\exists e_{i_n}\in S_{i,n},e_{i_n}\notin S_{i,1}$. Then every term except $\alpha_{i,1}$ is accompanied by either $e_{i_n}$ or ${e_{i_n}}^2$, since $p_{i,n}$ appears in every other term.

Above argument implies that every term is different from the rest. This implies that every distinct $\alpha_{i,j}$ is accompanied by a different monomial in the form $x_n^{n-1}x_j\prod_{k\in T_{i,j}}e_k$, for some $T_{i,j}\subset [\theta]$, where $T_{i,j}$ are distinct for different $i$. Thus, since at least one $\alpha_{i,j}$ is nonzero the resulting polynomial in the numerator is not identically zero. Then the numerator is a non-zero polynomial of the terms $\{x_1,x_2,\hdots,x_n,e_1,e_2,\hdots,e_n\}$. We know that the roots of a non-zero polynomial defined over a compact domain has Lebesque measure zero \cite{Caron2015}. Hence probability of numerator being zero is 0.

Then we have,
\begin{align*}
\pr{(\ref{eq:polynomial}) = 0} &= \pr{\text{ Numerator of (\ref{eq:polynomial}) $= 0$} \nonumber\\
&\text{ OR Denominator of (\ref{eq:polynomial}) $= \infty$}} \nonumber\\
&\leq \pr{\text{ Numerator of (\ref{eq:polynomial})} = 0} \nonumber\\
&+ \pr{\text{ Denominator of (\ref{eq:polynomial})} = \infty}.
\end{align*}
$\pr{\text{ Numerator of (\ref{eq:polynomial})} = 0} $ is shown to be zero by the argument above. We need to argue that the denominator cannot be infinity.

For this, denote denominator random variable to be $\xi$. We can write
\begin{equation}
\pr{\xi < \infty} = 1 - \pr{\lim\sup_{n\rightarrow\infty} \varepsilon_n},
\end{equation}
where $\varepsilon_n$ is the event that $\{\xi\geq t_n\}$ for a sequence of $t_n$ such that $\lim_{n\rightarrow \infty}t_n = \infty$. Pick $t_n = n^2$.

Since $\xi$ is nonnegative, we can apply Markov inequality to get
\begin{equation}
\pr{\varepsilon_n}\leq \frac{\mathbb{E}[\xi]}{t_n}.
\end{equation}

Clearly, $\mathbb{E}[\xi]<\infty$. Since $\sum_n \pr{\varepsilon_n} = \mathbb{E}[\xi]\sum_n \frac{1}{t_n} < \infty
$, applying Borel-Cantelli lemma, we have
\begin{equation}
\pr{\lim\sup_{n\rightarrow \infty} \varepsilon_n} = 0,
\end{equation}
which implies $\pr{\xi<\infty} = 1$.
Thus, we have 
\begin{equation}
\pr{(\ref{eq:polynomial}) = 0}\leq 0\Rightarrow \pr{(\ref{eq:polynomial}) = 0} = 0.
\end{equation}

\end{proof}

Now we can prove the main theorem:

Assume for the sake of contradiction that there exists a pair $(g,\tilde{E})$ that satisfy $X = g(Y,\tilde{E}),\tilde{E}\ci Y$.  By Lemma \ref{lem:characterization}, this is equivalent to the statement that there exists pair $(\mat{\tilde{M}},\mat{e})$ that satisfy (\ref{eq:reverseAppendix}). Define events $\varepsilon_1(\mat{\tilde{M}_k},\mat{\tilde{e}})=\{\text{ Event that $\mat{q=\tilde{M}_k\tilde{e}}$ }\}$ and  $\varepsilon_2(\mat{\tilde{M}_k})=\{\text{ Event that $\alpha_k\mat{q}=0$ }\}$, $\mat{M_k}$ is a fixed $\{0,1\}^{n^2\times m}$ matrix and $\alpha_k$ is one set of coefficients imposed by linear dependence of rows of $\mat{\tilde{M}_k}$. Notice that here $\mat{q}$ is a random vector, determined by $\{x_i\},i\in [n]$ and $\{e_i\},i\in [\theta]$. Clearly, $\varepsilon_1$ implies $\varepsilon_2$, thus $\pr{\varepsilon_1}\leq \pr{\varepsilon_2}$. Now we can write:
\begin{align}
&\pr{\exists (g,\tilde{E})\text{ such that }X = g(Y,E)} \nonumber\\
&= \pr{\exists (\mat{\tilde{M}},\mat{\tilde{e}})\text{ such that } \mat{q} = \mat{\tilde{M}\tilde{e}}}\label{eq:events1}\\
& = \pr{\exists (\mat{\tilde{M}},\mat{\tilde{e}})\text{ such that $\varepsilon_1$ is true} }\label{eq:events2}\\
& \leq \pr{\exists (\mat{\tilde{M}},\mat{\tilde{e}})\text{ such that $\varepsilon_2$ is true} }\label{eq:events3}\\
& = \pr{\exists (\mat{\tilde{M}})\text{ such that $\varepsilon_2$ is true} }\label{eq:events4}\\
&\leq \sum_k \pr{\text{ Given $\mat{\tilde{M}_k}, \alpha_k \mat{q}=0$} }\label{eq:events5}\\
& = \sum_k \pr{\alpha_k\mat{q} = 0} = 0.\label{eq:events6}
\end{align}
(\ref{eq:events1}) follows from the fact that both representations are equivalent by Lemma \ref{lem:characterization}. (\ref{eq:events3}) is due to the fact that if $\varepsilon_1$, then $\varepsilon_2$. (\ref{eq:events4}) is due to the fact that $\varepsilon_2$ does not depend on $\mat{\tilde{e}}$, but only on $\mat{\tilde{M}}$. (\ref{eq:events5}) follows from union bound over all matrices $\mat{\tilde{M}_k}$. The last equation follows from Lemma \ref{lem:poly} and the fact that there are finitely many $\mat{\tilde{M}_k}\in \{0,1\}^{n^2\times m}$ with $m<n(n-1)$ columns.
\qedS
\subsection{A counterexample when \texorpdfstring{$S_{i,j} = S_{i,k}$}{Si,j=Si,k}}
Following counterexample shows that without the additional assumption, identifiability result in Theorem \ref{thm:main} does not hold.

Consider the following equation:
\begin{align}
\begin{pmatrix}
    p_{1,1}  \\
    p_{2,1} \\
    p_{3,1} \\
    p_{1,2} \\
    p_{2,2} \\
    p_{3,2} \\
    p_{1,3} \\
    p_{2,3}\\
    p_{3,3}
\end{pmatrix}
=
\begin{pmatrix}
    1 & 0 & 0  \\
    0 & 1 & 0  \\
    0 & 0 & 1  \\
    1 & 0 & 0  \\
    0 & 1 & 0 \\
    0 & 0 & 1 \\
    1 & 0 & 0 \\
    0 & 1 & 0\\
    0 & 0 & 1 
\end{pmatrix}
\times
\begin{pmatrix}
    e_1 \\
    e_2 \\
    e_3 
\end{pmatrix}
\end{align}

In the reverse direction, we can fit the following system, which has smaller cardinality for the exogenous variable:

\begin{align}
\begin{pmatrix}
    q_{1,1}  \\
    q_{2,1} \\
    q_{3,1} \\
    q_{1,2} \\
    q_{2,2}\\
    q_{3,2} \\
    q_{1,3} \\
    q_{2,3}\\
    q_{3,3} 
\end{pmatrix}
=
\begin{pmatrix}
    1 & 0 & 0 \\
    0 & 1 & 0 \\
    0 & 0 & 1 \\
    1 & 0 & 0 \\
    0 & 1 & 0 \\
    0 & 0 & 1 \\
    1 & 0 & 0 \\
    0 & 1 & 0 \\
    0 & 0 & 1 
\end{pmatrix}
\times
\begin{pmatrix}
    \frac{x_1}{x_1+ x_2 + x_3 } \\
    \frac{x_2}{x_1+ x_2 + x_3} \\
    \frac{x_3}{x_1+ x_2 + x_3} 
\end{pmatrix}
\end{align} 

Notice that $e_i$ terms completely disappear in this symmetric case. Thus, exogenous variable with $n$ states is sufficient to describe the reverse conditional probability distribution matrix, independent from the cardinality of exogenous variable in the true direction, i.e., the value of $\theta$.

Our main theorem suggests, under the condition that no $S_{i,j}$ is the exact subset of $\{e_1,e_2,...\}$, no such case can arise, and the reverse direction requires exogenous variable with at least $n(n-1)$ states.

\subsection{A counterexample when \texorpdfstring{$p_{i,j}\geq 0$}{pi,j>0}}
The critical component of the proof was that each linear equation implied by the rank deficiency of the system had unique non-zero coefficients. Here, we provide a counterexample to the theorem when this condition is violated.

Consider the following system:
\begin{align}
\begin{pmatrix}
    p_{1,1}  \\
    p_{2,1} \\
    p_{1,2} \\
    p_{2,2} 
\end{pmatrix}
=
\begin{pmatrix}
    1 & 1 & 1 & 1\\
    0 & 0 & 0 & 0\\
    1 & 1 & 0 & 0\\
    0 & 0 & 1 & 1
\end{pmatrix}
\times
\begin{pmatrix}
    e_1 \\
    e_2 \\
    e_3 \\
    e_4
\end{pmatrix}
\end{align}
 
In the reverse direction, we can fit the following system, which has smaller cardinality for the exogenous variable.

\begin{align}
\begin{pmatrix}
    q_{1,1}  \\
    q_{2,1} \\
    q_{1,2} \\
    q_{2,2} 
\end{pmatrix}
=
\begin{pmatrix}
    1 & 0 \\
    0 & 0 \\
    0 & 1 \\
    1 & 1  
\end{pmatrix}
\times
\begin{pmatrix}
    \frac{x_1}{x_1+x_2(e_1+e_2)} \\
    \frac{x_2(e_1+e_2)}{x_1+x_2(e_1+e_2)} \\
\end{pmatrix}
\end{align} 

Notice that this is true independent of the selection of $e_i$, hence the theorem cannot be extended to case where $p_{i,j}= 0$ is allowed.

\subsection{Proof of Theorem \ref{thm:cardNPhard}}
We first define \emph{decomposition problem}, and show it is $\NP$ hard. 
\begin{definition}
\emph{Decomposition problem:} For a given nonnegative matrix $M$, with column sums equal to 1, consider the decomposition $\sum_{x\in\mathcal{X}}xF_x$, where $F_x$ are 0,1 matrices with single 1 per column, and $\sum_{x\in\mathcal{X}}x = 1$ with $x\geq 0$. Identify the decomposition that minimizes $\card(\mathcal{X})$.
\end{definition}

\begin{lemma}
\label{lem:decompositionHardness}
The decomposition problem is $\NP$ hard.
\end{lemma}
\begin{proof}
We use subset sum problem for the reduction:
\begin{definition}
\emph{Subset sum problem:} For a given set of integers $V$, and an integer $a$, decide whether there exists a subset $S$ of $V$ such that $\sum_{u\in S}u = a$. 
\end{definition}

Subset sum is a well known NP complete problem. Consider any instance of the subset sum problem with the set $V = \{u_1,u_2,...,u_m\}$ and an integer $a$. Assume without loss of generality $u_i\neq0,\forall i\in V$. If not, one can work with the set of nonzero values in $V$. Construct the $m$ by $2$ matrix $\mat{M}$ with $\mat{M}(i,1)=u_i$ and $\mat{M}(1,2) = a, \mat{M}(2,2) = -a+\sum_{i\in[m]}u_i, \mat{M}(i,2)=0,\forall i\in\{2,3,...,m\}$.

Update $\mat{M}$ by dividing each element by $\sum_{i\in V}u_i$. This does not change the answer to the subset sum problem. Now column sum of $\mat{M}$ is 1 for both columns. It can be shown that the decomposition of Lemma \ref{lem:decompUpper} can always be applied here to get a decomposition with $|\mathcal{X}|=m+1$ (number of non-zero terms $- 1$). We also know that any decomposition need to touch each nonzero element in each column at least once. Hence the column with largest number of nonzero elements yields the lower bound $|\mathcal{X}^*|\geq m$.

We show that optimal decomposition size is $m$ if and only if there is a subset of $V$ that sums to $a$. 

First, assume $\exists S\subset V$ such that $\sum_{u\in S}u = a$. Consider the following decomposition: Let $x_k = u_k$ for $k\in[m]$. Let $F_k(k,1)=1$ be the only nonzero element in the first column. Pick the nonzero element in the second column of $F_k$ based on the membership of $u_k$ in $S$: Let $F_k(1,2) = 1$ if $u_k\in S$ and $F_k(2,2) = 1$ if $u_k\in V\backslash S$. This decomposition is optimal due to lower bound.

Now we show that if optimum has size $m$, then $\exists$ a subset sum with value $a$: Recall that each $F_k$ has a single 1 per column and decomposition has $m$ terms. Since first column of $\mat{M}$ has $m$ non-zero terms, each term in the decomposition must be equal to the elements of this column, which is the elements of the given set. Since $F_k$ are 0,1 matrices with single 1 per column by construction, every element of $M$ must be a subset sum of the decomposition terms. Hence, $a$ is a subset sum of set elements.

This shows that a subset sum exists if and only if optimal decomposition has size $m$. Thus, if we could find the optimal decomposition size in all instances of the decomposition problem, we would solve all instances of the subset sum problem.
\end{proof}

Now, we give a definition related to decomposability:
\begin{definition}
\emph{$\alpha-$decomposability:} A matrix $M$ is called $\alpha-$decomposable if it can be written as a convex combination of $\alpha$ $\{0,1\}$ matrices, each with a single 1 per column.
\end{definition}

Identifying the exogenous variable with minimum cardinality is equivalent to finding minimum size $\mat{e}$ with $\mat{M}\in \mathcal{C}$ such that $vec(\mat{Y|X}) = \mat{Me}$ from Lemma \ref{lem:characterization}. Notice that matrix $\mat{Y|X}$ is $\alpha-$decomposable if and only if there exists $\mat{e}$ of size $\alpha$ along with $\mat{M}\in\mathcal{C}$ such that $vec(\mat{Y|X}) = \mat{Me}$. 
 
Consider any decomposition problem. If we had an algorithm that could solve all instances of the problem of identifying $E$ with minimum support, we could feed the normalized matrix from the decomposition problem as $\mat{Y|X}$ to this algorithm and solve the decomposition problem.

\subsection{Proof of Theorem \ref{thm:equivalencetoentropy}}
Consider $m$ random variables $\{U_1,U_2,\hdots,U_m\}$ each with $n$ states. Let $\{p_1,p_2,\hdots,p_m\}$ stand for the marginal probability distributions of each random variable. Consider the following optimization problem:
\begin{equation}
\begin{aligned}
H^*(U_1,U_2,...,U_m) = & \underset{p(u_1,u_2,...,u_m) }{\text{min}}  & & H(U_1,U_2,...,U_m) \\
& \text{s. t.}
& &p(u_i) = p_i,\forall i \label{eq:minentropyoptimization}
\end{aligned}
\end{equation}

In words, optimization problem finds the joint distribution of $m$ variables with minimum entropy, subject to marginal distribution constraints for every variable. Notice that this problem is non-convex: It minimizes a concave function with respect to linear constraints.

Let $X, E$ be discrete, independent random variables where $X\in[m]$ and $E\in[t]$. Recall that every conditional distribution $\pr{Y|X=x}$ can be written as the distribution of a function of random variable $E$. Let us investigate the underlying probability space, in order to generate the probability space to optimize the joint distribution over.

We can consider the product probability space 
\begin{equation}
\mathcal{P} = (\Omega = \Omega_X\times\Omega_E = [m]\times [t]\hspace{0.1in},\hspace{0.1in}\mathcal{F} = 2^{\Omega}\hspace{0.1in},\hspace{0.1in}p)
\end{equation}
for random variables $X:\Omega_X\rightarrow [n], E:\Omega_E\rightarrow [t]$ with $X(\omega) = \omega, Y(\omega) = \omega$. Then for $\omega = (\omega_x,\omega_e)\in \Omega$, $p(\omega)= p_X(\omega_x)p_E(\omega_e)$.

Consider the equation $Y = f(X,E)$. Let $f_x:\Omega_E\rightarrow \Omega_Y$ be the function mapping $E$ to $Y$ when $X=x$, i.e., $f_x(E) \coloneqq f(x,E)$. Then \begin{align}
\pr{Y=y|X=x} &= \pr{f_x(E) = y|X = x} \\
&= \pr{f_x(E)=y},
\end{align}
where last equality follows from the fact that $X\ci E$. 

Let sigma algebra generated by the random variable $f(x,E)$ be $\mathcal{F}_{\omega_x}$. Formally, the sigma algebras generated by $f(x,E)$ are subsigma algebras of disjoint, but identical sigma algebras. In other words, even though $\mathcal{F}_{\omega_x}\subseteq 2^{(\omega_x,\Omega_E)}$ are disjoint for different $\omega_x$, $(\omega_x,\Omega_E)$ are identical for every $\omega_x\in\Omega_X$. Thus the sigma algebras generated by $f(x,E) = g_x(E)$ can be thought of as subsigma algebras of the same sigma algebra, i.e., the one generated by $E$. In other words, we can equivalently construct $\mathcal{F}_{\omega_x} \subseteq 2^{\Omega_E}$. This construction allows us to talk about the joint probability distribution of the random variables $\{f_x(E):x\in\mathcal{X}\}$.

Let $U_i = f_i(E)$. Then we have,
\begin{align}
H(E) &= H(E|U_1,U_2,...,U_m)+H(U_1,U_2,...,U_m) \\
		& \hspace{1in} - H(U_1,U_2,...,U_m|E)\\
&\geq H(U_1,U_2,...,U_m). 
\end{align}
Inequality follows from the fact that $H(U_1,U_2,...,U_m|E) = 0$ and $H(E|U_1,U_2,...,U_m)\geq 0$.

$H(E)\geq H(U_1,U_2,...,U_m)$. 
Thus, the best lower bound on the entropy of exogenous variable $E$ can be obtained by solving the optimization problem below. 
\begin{equation}
\begin{aligned}
H(E)\geq & & \underset{p(u_1,u_2,...,u_m) }{\text{min}}
& & &H(U_1,U_2,...,U_m) \\
& & \text{subject to}
& & & p(u_i) = p_i, \; i = 1, \ldots, m.
\end{aligned}
\end{equation}

Since we are picking $E$ with minimum possible entropy without restricting (not observing) the functions $f_x$, we can actually construct an $E$ that achieves this minimum: Let the optimal joint distribution be $p^*(u_1,u_2,...,u_m)$. We can construct $E$ with $n^m$ states with state probabilities equal to $p^*(u)$ for each configuration of $u$. This $E$ has the same entropy as the joint entropy, since probability values are the same. Thus,
\begin{equation}
\begin{aligned}
H(E^*) = & & \underset{p(u_1,u_2,...,u_m) }{\text{min}}
& & &H(U_1,U_2,...,U_m) \\
& & \text{subject to}
& & & p(u_i) = p_i\nonumber
\end{aligned}
\end{equation}

The functions $f_i$, where $U_i$ has the same distribution as $f_i(E)$, can be constructed from the distribution of $E$ and $U_i$. This determines the function $f$, hence the causal model $\mathcal{M}$ that induces the conditional distribution $\mat{Y|X}$. Note that $E\ci X$ by construction. (For a complete argument on how one can always generate $E\ci X$ based on a set of samples of $X,Y$, see the proof of Lemma \ref{lem:characterization}).

\subsection{Proof of Corollary \ref{cor:entropyNPhard}}
Assume there exists a black box that could find the causal model with minimum entropy exogenous variable $E$. Consider an arbitrary instance of the problem of minimizing the joint entropy of a set of variables $\{U_1,U_2,\hdots,U_n\}$ subject to marginal constraints. Construct matrix $\mat{M} = [p_1, p_2, \hdots, p_n]$, where $p_i$ is the distribution of variable $U_i$ as a column vector. Feed $\mat{M}$ into this black box as a hypothetical conditional distribution $\mat{Y|X}$. From the proof of Theorem \ref{thm:equivalencetoentropy}, $H(E)$ output by this black box gives the minimum entropy joint distribution of the variables $U_i$. Hence, finding the causal model with minimum entropy would give the solution to this $\NP$ hard problem.

\subsection{Proof of Proposition \ref{prop:entropybasedalgorithm}}
Given a joint distribution $p(x,y)$, consider the following algorithm:
In stage 1, feed the set of conditional distributions $\{\pr{Y|X=i}:i\in[n]\}$ to algorithm $\mathcal{A}$ to obtain $E$ with minimum entropy. Algorithm outputs a variable $E$ along with $\{f_1,f_2,\hdots,f_n\}$ such that the distribution of $f_j(E)$ is the same as the conditional distribution of $Y$ given $X = j,\forall j$. This set of $f_j$ determine $f$ where $Y = f(X,E),E\ci X$. Since algorithm optimizes entropy of $E$, $H(E)\leq H(E_0)$. In stage 2, feed the conditional distributions  $\{\pr{X|Y=i}:i\in[n]\}$ to algorithm $\mathcal{A}$ to obtain $\tilde{E}$. From the conjecture any $\tilde{E}$ satisfies $H(X)+H(E_0)<H(Y)+H(\tilde{E})$. Since $H(E)\leq H(E_0)$, we have $H(X)+H(E)<H(Y)+H(\tilde{E})$, which can be used for identifiability.

\subsection{ Greedy Entropy Minimization Outputs a Local Optimum}
\begin{proposition}
\label{prop:greedyLocalOptima}
For two variables, Algorithm \ref{alg:heuristic} always returns a local optimum.
\end{proposition}

Consider random variables $U,V$ with marginal distributions $p_u,p_v$. We first show that the KKT conditions on the problem (\ref{eq:minentropyoptimization}) imply that the optimal solution is \emph{quasi-orthogonal}:
\begin{equation}
p^*(u,v) =\mat{M} = \mat{U}(x,y)\delta_{x,y},
\end{equation}
where $\mat{U}$ is rank 1 and $\delta_{u,v}$ is an indicator for the support of optimal joint distribution. We call such $\mat{M}$ as masked submatrix of a rank 1 matrix.

Let $(i,j)$th entry of the joint distribution be $x_{i,j}$. We have $n^2$ variables $\{x_{i,j},i\in [n],j\in [n]\}$ to  optimize over. In a general optimization problem
\begin{equation}
\begin{aligned}
 & \underset{x }{\text{min}}
 & &f_0(x)\\
 & \text{s. t.}
 & & h_i(x) = 0, i\in[p] \nonumber\\
 & &  & f_i(x)\leq 0, i\in [m],
\end{aligned}
\end{equation}

Lagrangian becomes 
\begin{equation}
L(x,\lambda,v) = f_0(x)+\sum_{i=1}^m\lambda_i f_i(x)+\sum_{i=1}^pv_ih_i(x),
\end{equation}
which gives the KKT conditions
\begin{align*}
&f_i(x^*)\leq 0, i\in [m]\\
&h_i(x^*) = 0, i\in [p]\\
&\lambda_i^*\geq 0, i\in[m]\\
&\lambda_i^*f_i(x^*) = 0, i\in [m]\\
&\nabla L(x^*,\lambda^*,v^*) = 0
\end{align*}

This implies, for fixed $i$, either $f_i(x^*) = 0$ or $\lambda_i^*=0$. The optimization problem we have is
\begin{equation}
\begin{aligned}
 & \underset{x_{i,j} }{\text{min}}
 & &\sum_{i,j} -x_{i,j}\log{x_{i,j}} \\
 & \text{s. t.}
 & & \sum_j{x_{i,j}} = p_u(i), \forall i,\hspace{0.2in} \sum_i{x_{i,j}} = p_v(j), \forall j \nonumber\\
 & &  & x_{i,j}\geq 0, \forall i,j.
\end{aligned}
\end{equation}

Substituting corresponding $f_i,h_i$, we get the following conclusion: At optimal point $x_{i,j}$, either $x_{i,j}=0$, or $1-\log{x_{i,j}}+v_i^{(1)}+v_j^{(2)}=0$. This implies that $x_{i,j} = 2^{1+v_i^{(1)}+v_j^{(2)}}$. Hence the optimal joint satisfies $p_{i,j}^* = u_iv_j$ for some vectors $u,v$, whenever $p_{i,j}^*\neq 0$.

Next we show that such $u,v$ can be constructed from any algorithm output:
\begin{theorem}
For any matrix $\mat{M}$ output by Algorithm \ref{alg:heuristic}, there exists $u,v$ where $\mat{M}$ is a masked submatrix of $uv^T$.
\end{theorem}
\begin{proof}[Proof of Proposition \ref{prop:greedyLocalOptima}]
Consider the following variant of Algorithm \ref{alg:heuristic} for two distributions $\mat{p},\mat{q}$: Initialize an $n \times n$ zero matrix $\mat{M}_0$. In every iteration, find $m_1 = \max_i{\mat{p}(i)}, m_2 = \max_i{\mat{q}(i)}$. Let $a = \argmax_i{\mat{p}(i)}$ and $b = \argmax_i{\mat{q}(i)}$. Assign $r = \min\{m_1,m_2\}$ to the $(a,b)$th entry of $\mat{M}$. Update $\mat{p}(a) \leftarrow \mat{p}(a)-r,\mat{q}(b) \leftarrow \mat{q}(b)-r$. Repeat the process until $\mat{p}=0,\mat{q}=0$.

This variant constructs the joint distribution matrix rather than the distribution of $E$ directly. We need the following definitions:
\begin{definition}
An ordered set of coordinates $((i_1,j_1),(i_2,j_2),\hdots, (i_m,j_m))$ of a matrix $\mat{M}$ is called a path on matrix $\mat{M}$, if either $ i_{t+1} = i_t$ or $ j_{t+1} = j_t, \forall t\in [m-1]$.
\end{definition}
\begin{definition}
An ordered set of coordinates $((i_1,j_1),(i_2,j_2),\hdots, (i_m,j_m))$ of a matrix $\mat{M}$ is called a cycle on matrix $\mat{M}$, if either $ i_{t+1} = i_t$ or $ j_{t+1} = j_t, \forall t\in [m-1]$ and either $i_1 = i_m$ or $j_1 = j_m$.
\end{definition}

First we prove a structural characterization for matrix $\mat{M}$.
\begin{lemma}
Any matrix $\mat{M}$ output by Algorithm $\ref{alg:heuristic}$ contains no cycles.
\end{lemma}
\begin{proof}
Consider a bipartite graph $G$ with $n$ left and $n$ right vertices constructed as follows: $G$ has an edge $(i,j)$ if and only if $\mat{M}(i,j)$ is nonzero.

It is easy to see that the matrix $\mat{M}$ has no cycles if and only if the corresponding bipartite graph $G$ has no cycles. We claim that, for an $\mat{M}$ output by Algorithm $\ref{alg:heuristic}$, corresponding $G$ cannot have any cycles.

Construct the bipartite graph in the same order as matrix $\mat{M}$ is created by Algorithm $\ref{alg:heuristic}$. Notice that when the algorithm assigns a value to $\mat{M}(i,j)$, it satisfies either $i$th row constraint or $j$th column constraint. Then no other value can be assigned to that row or column. We call this zeroing out the corresponding row/column.

In the bipartite representation, say an edge $(i,j)$ is added at time $t$. We call a vertex of the added edge \emph{closed}, if the corresponding dimension (row or column) is zeroed out. Thus each added edge to the bipartite graph closes one of the endpoints of that edge. A closed vertex cannot participate to the formation of any other edges. This captures the fact that when zeroed out, a row/column cannot be assigned a non-zero value again by the algorithm.

Assume at time $t$, a cycle is formed in the bipartite graph. Then at time $t-1$, there must be a path with a single edge missing. Let the edges of this path be labeled as $\{e_1,e_2,\hdots e_m\}$. There is a time order in which these edges are constructed. Let $e_k$ be the first edge formed in this path. Then one of its endpoints must be closed. However it belongs to a path, which means an edge was attached to a closed vertex, which is a contradiction. Assume $e_k$ is the edge at the end of the path and the closed endpoint is also the endpoint of the path. However an edge is attached to this closed vertex at time $t$ to form the cycle, which is a contradiction.
\end{proof}

Consider a matrix $\mat{M}$ output by the algorithm. We prove that $\mat{M}$ is a masked submatrix of $uv^T$ by construction: 

Let the first entry selected by Algorithm \ref{alg:heuristic} have coordinates $(i_0,j_0)$. Since algorithm zeroes out either the row or column, $(i_0,j_0)$ is the only non-zero entry in either its column or row. Assume without loss of generality that selection of $(i_0,j_0)$ by Algorithm \ref{alg:heuristic} zeroes out row $i_0$.

Initialize by assigning $u_{i_0} = 1$ and $v_{j_0} = \mat{M}(i_0,j_0)$.  

Let $S_{j_0}$ represent the non-zero row indices for column $j_0$. Now assign the values of $u(k)$ for $k\in S_{j_0}$ such that $u(S_{j_0})v_{j_0} = \mat{M}(S_{j_0},j_0)$, where $u(S)$ stands for the subvector containing entries with index in $S$.

Repeat the above procedure by exhausting either a row or column at each time. 

\begin{lemma}
The above construction never runs into an entry $(i,j)$ for which both $u_i$ and $v_j$ were assigned before.
\end{lemma}
\begin{proof}
Assume otherwise. Then, due to the process of assigning the entries of $u,v$, there must be two paths from $(i,j)$ to the starting point of $(i_0,j_0)$: One path that follows column $j$, and one path that follows row $i$. In other words, there exists $k,l\in [n]$ such that there are two paths $((i,j),(k,j),\hdots,(i_0,j_0))$ and $((i,j),(i,l),\hdots,(i_0,j_0))$. Combining these paths, we get the cycle $((i,j),(k,j),\hdots,(i_0,j_0),\hdots,(i,l))$, which is a contradicton.
\end{proof}

Hence, algorithm selects every element in $u$ and $v$ at most once. Thus it produces a valid assignment for $u,v$.
\end{proof}

\subsection{Implementation Details and Sampling Distributions with Diverse Entropy Values}
We implemented our algorithms in MATLAB. In order to sample from a wide range of entropy values, when we need to sample a distribution from the $n-1$ dimensional simplex, we generate $n$ independent identically distributed log Gaussian random variables with parameter $\sigma$. For verifying the conjecture on artificial data, we generate the distributions for $E$ swiping the $\sigma$ parameter from 2 to 8, taking only the integer values. The distribution for $X$ is uniformly randomly sampled over the simplex.

For the true causal direction, the function $f$ is sampled as follows: We generate $\theta$ matrices $F_i$, where each has a single 1 per column, randomly selected out of all $n$ rows. Each $F_e$ represent the function $f_{e}(X)\coloneqq f(X,e)$. Together with $\mat{e}$, this determines the conditional probability distribution matrix $\mat{Y|X}$.

The number of states for quantization is chosen as $n=\min\{N/10,512\}$, where $N$ is the number of samples for that particular cause-effect pair. Hence, we pick $n$ to assure each state has at least 10 samples on average. An upper bound of 512 is used to limit the computational complexity.
\newpage

\end{document}